\documentclass[10pt,twocolumn,letterpaper]{article}

\usepackage[pagenumbers]{cvpr} 

\usepackage{bm}
\usepackage{amsthm}
\usepackage{graphicx}
\usepackage{algorithm}
\usepackage{algorithmic}
\usepackage{multirow}
\usepackage{multicol}

\newtheorem{proposition}{Proposition}
\newcommand{\myparagraph}[1]{\noindent\textbf{#1}~}

\usepackage[dvipsnames]{xcolor}

\definecolor{cvprblue}{rgb}{0.21,0.49,0.74}
\usepackage[pagebackref,breaklinks,colorlinks,citecolor=cvprblue]{hyperref}

\def\paperID{10063} 
\def\confName{CVPR}
\def\confYear{2024}

\title{InfLoRA: Interference-Free Low-Rank Adaptation for Continual Learning}

\author{%
Yan-Shuo Liang and Wu-Jun Li\thanks{Wu-Jun Li is the corresponding author.} \\
National Key Laboratory for Novel Software Technology, \\
Department of Computer Science and Technology, Nanjing University, P. R. China\\
\tt{\small liangys@smail.nju.edu.cn},\texttt{\small liwujun@nju.edu.cn} 
}

\begin{document}
\maketitle
\begin{abstract}
    Continual learning requires the model to learn multiple tasks sequentially. In continual learning, the model should possess the ability to maintain its performance on old tasks~(stability) and the ability to adapt to new tasks continuously~(plasticity). Recently, parameter-efficient fine-tuning~(PEFT), which involves freezing a pre-trained model and injecting a small number of learnable parameters to adapt to downstream tasks, has gained increasing popularity in continual learning. Although existing continual learning methods based on PEFT have demonstrated superior performance compared to those not based on PEFT, most of them do not consider how to eliminate the interference of the new task on the old tasks, which inhibits the model from making a good trade-off between stability and plasticity. In this work, we propose a new PEFT method, called \underline{in}terference-\underline{f}ree \underline{lo}w-\underline{r}ank \underline{a}daptation~(\mbox{InfLoRA}), for continual learning. \mbox{InfLoRA} injects a small number of parameters to reparameterize the pre-trained weights and shows that fine-tuning these injected parameters is equivalent to fine-tuning the pre-trained weights within a subspace. Furthermore, \mbox{InfLoRA} designs this subspace to eliminate the interference of the new task on the old tasks, making a good trade-off between stability and plasticity. Experimental results show that \mbox{InfLoRA} outperforms existing state-of-the-art continual learning methods on multiple datasets. Code is available at {\small\url{https://github.com/liangyanshuo/InfLoRA}}.
\end{abstract}

\section{Introduction}
\label{sec:intro}
Continual learning requires the model to learn multiple tasks sequentially~\cite{parisi2019continual}. To achieve continual learning, the model should possess two essential abilities, including the ability to keep its performance on the old tasks~(stability) and the ability to adapt to the new tasks continuously~(plasticity)~\cite{parisi2019continual}. Furthermore, two different scenarios are often considered in continual learning, including task-incremental scenario~\cite{masana2021importance} and class-incremental scenario~\cite{wang2023comprehensive}. Task-incremental scenario allows the model to get task identities during inference. On the contrary, class-incremental scenario does not allow the model to get task identities during inference, making the model learn to distinguish all the classes across all the tasks.



Recently, parameter-efficient fine-tuning~(PEFT)~\cite{hu2021lora,houlsby2019parameter,DBLP:conf/eccv/JiaTCCBHL22}, which involves freezing a pre-trained model and injecting a small number of learnable parameters to adapt to downstream tasks, has gained increasing popularity in continual learning~\cite{wang2022learning,smith2023coda,gao2023unified}, especially in the class-incremental scenario. More specifically, existing continual learning methods based on PEFT~\cite{khan2023introducing,DBLP:conf/eccv/0002ZESZLRSPDP22} inject the learnable parameters into a pre-trained model using some popular PEFT methods such as prompt-tuning~\cite{DBLP:conf/emnlp/LesterAC21} or low-rank adaptation~(LoRA)~\cite{hu2021lora}. Subsequently, these methods freeze the pre-trained weights and sequentially fine-tune the injected parameters on multiple tasks throughout the continual learning process.

Although continual learning methods based on PEFT have demonstrated superior performance compared to those not based on PEFT~\cite{wang2022learning}, most of them do not consider how to eliminate the interference of the new task on the old tasks, which inhibits the model from making a good trade-off between stability and plasticity.
Specifically, when learning a new task, existing continual learning methods based on PEFT either reuse the previously learned parameters to adapt to the new task~\cite{wang2022learning,gao2023unified} or randomly expand some parameters first and then adapt to the new task~\cite{smith2023coda,DBLP:conf/eccv/0002ZESZLRSPDP22,wang2022s}. 
During this process, the interference of the new task on the old tasks exists due to the shared parameters between new and old tasks, which means fine-tuning a pre-trained model on a new task may interfere with the model's performance on the old tasks. As a result, it is hard for the model to make a good trade-off between stability and plasticity.


In this work, we propose a new PEFT method, called \underline{in}terference-\underline{f}ree \underline{lo}w-\underline{r}ank \underline{a}daptation~(\mbox{InfLoRA}), for continual learning. The contributions of this work are listed as follows:
\begin{itemize}
  \item \mbox{InfLoRA} injects a small number of parameters to reparameterize the pre-trained weights and shows that fine-tuning these injected parameters is equivalent to fine-tuning the pre-trained weights within a subspace.
  \item \mbox{InfLoRA} designs this subspace to eliminate the interference of the new task on the old tasks, making a good trade-off between stability and plasticity.
  \item Experimental results show that \mbox{InfLoRA} outperforms existing state-of-the-art continual learning methods on multiple datasets.
\end{itemize}


\section{Related Work and Preliminaries}
\subsection{Related Work}
\myparagraph{Parameter-Efficient Fine-Tuning} 
Parameter-efficient fine-tuning~(PEFT) methods freeze a pre-trained model and inject a small number of learnable parameters to adapt to downstream tasks. In this way, PEFT methods
reduce the inefficiency of full fine-tuning methods which fine-tune all the parameters of a pre-trained model to learn downstream tasks. 
For example, Adapter~\cite{houlsby2019parameter} adds small modules in different layers of Transformers and only tunes these added modules to adapt to downstream tasks. Prompt-tuning~\cite{DBLP:conf/emnlp/LesterAC21} and Prefix-tuning~\cite{DBLP:conf/acl/LiL20} insert a set of learnable tokens into the input of the Transformer layers and only tune these tokens to adapt to downstream tasks. Low-rank adaptation~(LoRA)~\cite{hu2021lora} reparameterizes the pre-trained weights with low-rank branches and only tunes these branches to adapt to the downstream tasks. Although these methods tune much fewer learnable parameters than full fine-tuning, they always show comparable or even superior performance compared with full fine-tuning~\cite{zaken2022bitfit,fu2022adapterbias,hu2021lora,DBLP:conf/nips/MahabadiHR21}. Early PEFT methods focus on natural language processing~(NLP). 
Recently, PEFT methods have also been proposed for computer vision~(CV). For example, visual prompt tuning~(VPT)~\cite{DBLP:conf/eccv/JiaTCCBHL22} and AdapterFormer~\cite{chen2022adaptformer} apply prompt-tuning and Adapter techniques to CV tasks, respectively. Both of them exhibit comparable performance to full fine-tuning.

\myparagraph{Continual Learning}
Early continual learning was usually considered in the context of learning from scratch. Three types of continual learning methods are proposed, including regularization-based methods~\cite{zenke2017continual,jung2020continual,aljundi2018memory, kirkpatrick2017overcoming}, memory-based methods~\cite{DBLP:conf/nips/AljundiBTCCLP19,chrysakis2020online,DBLP:conf/nips/AljundiLGB19,sun2022exploring,DBLP:conf/nips/LiangL23}, and expansion-based methods~\cite{rusu2016progressive,DBLP:conf/nips/Hung0WCCC19,li2019learn}. Regularization-based methods employ a penalty loss~(regularization) to prevent important parameters of old tasks from changing too much. Memory-based methods maintain a memory buffer to store information about old tasks. Expansion-based methods dynamically expand the model's architecture for each new task. 

Recently, with the advancements of pre-trained models~\cite{he2022masked,dosovitskiy2020image,DBLP:conf/naacl/DevlinCLT19}, using pre-trained models for continual learning has gained increasing popularity. Some continual learning methods fully fine-tune the pre-trained models~\cite{boschini2022transfer,zheng2023preventing}, which has been shown to be inefficient. Other methods explore PEFT methods in continual learning. For instance, some existing continual learning methods~\cite{smith2023coda,wang2022learning,khan2023introducing,DBLP:conf/eccv/0002ZESZLRSPDP22} introduce prompt-tuning in continual learning, achieving much higher performance than previous methods that learn from scratch, especially in the class-incremental scenario. The method in~\cite{gao2023unified} introduces a framework in continual learning that can be combined with many existing PEFT methods, such as prompt-tuning, LoRA and Adapter. However, all these methods do not consider how to eliminate the interference of the new task on the old tasks, which inhibits the model from making a good trade-off between stability and plasticity.
\subsection{Preliminaries}
We first introduce low-rank adaptation~(LoRA)~\cite{hu2021lora}, a popular PEFT method related to our method. Then, we give the problem definition for continual learning. 

\myparagraph{Low-Rank Adaptation} 
LoRA~\cite{hu2021lora} is one of the most popular PEFT methods. It assumes that the changes of parameters lie in a low-rank space when the model is fully fine-tuned on a downstream task. Specifically, for a linear layer with the input dimension $d_{I}$ and the output dimension $d_{O}$, we represent its weight with $\bm{W}^{d_{O}\times d_{I}}$. Then, LoRA reparametrizes the pre-trained weight $\bm{W}$ by expanding a branch with two matrices, $\bm{A}\in \mathbb{R}^{d_{O}\times r}$ and $\bm{B}\in\mathbb{R}^{r\times d_{I}}$. Typically, $r$ is much smaller than the input dimension $d_{I}$ and output dimension $d_{O}$, making $\bm{A}$ a dimensionality increasing matrix and $\bm{B}$ a dimensionality reduction matrix. Finally, LoRA modifies the forward propagation in this linear layer as
$\bm{e} = \bm{W}\bm{h} + \bm{A}\bm{B}\bm{h}$. 
Here, $\bm{h}$ and $\bm{e}$ denote the input and output of this layer, respectively. LoRA initializes $\bm{A}$ as $\bm{0}$ and initializes $\bm{B}$ using Gaussian distribution. During the learning of the downstream tasks, LoRA freezes the pre-trained weight $\bm{W}$ and only fine-tunes the parameters $\bm{A}$ and $\bm{B}$.

\begin{figure*}[t]
  \setlength{\belowcaptionskip}{-0.5cm}
  \centering
  \includegraphics[width=\textwidth]{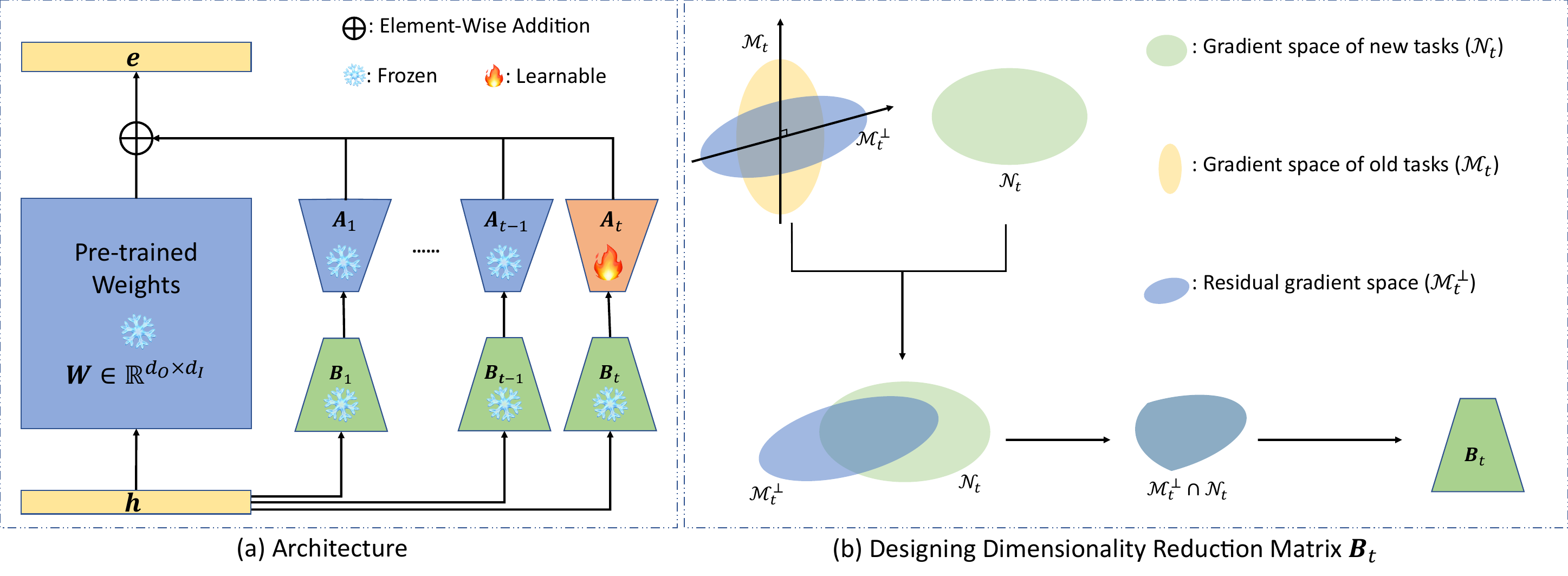}
 \captionsetup{skip=2pt}
 \caption{(a) The architecture of our \mbox{InfLoRA} in a certain linear layer of a Transformer. During the learning of the $t$-th task, the pre-trained weight and all the old branches are frozen, and only $\bm{A}_{t}$ is fine-tuned.~(b) The pipeline of designing dimensionality reduction matrix $\bm{B}_{t}$.}
  \label{fig:method}
\end{figure*}

\myparagraph{Problem Definition}
In continual learning, there is a sequence of tasks with different distributions. We define the task sequence as $\mathcal{D}=\{\mathcal{D}_{1},...,\mathcal{D}_{T}\}$, where the $t$-th task $\mathcal{D}_{t}=\{(\bm{x}_{i,t},y_{i,t})\}_{i=1}^{n_{t}}$. Here, $\bm{x}_{i,t}$ denotes an input sample and $y_{i,t}$ denotes its label. The objective of continual learning is to train a model sequentially on these tasks and ensure that the model performs well on all of them.

We follow existing continual learning methods~\cite{DBLP:conf/eccv/0002ZESZLRSPDP22,wang2022learning} based on PEFT and assume the model is a pre-trained Vision Transformer~(ViT)~\cite{dosovitskiy2020image}. 
Specifically, assume the model is $h_{\bm{\Phi}}(f_{\bm{\Theta}}(\cdot))$ where $h_{\bm{\Phi}}(\cdot)$ is the classifier with parameters $\bm{\Phi}$ and $f_{\bm{\Theta}}(\cdot)$ is the pre-trained ViT backbone with pre-trained parameters $\bm{\Theta}$.
Similar to existing work~\cite{DBLP:conf/eccv/0002ZESZLRSPDP22}, our focus is primarily on the class-incremental scenario, where task identities are unknown during inference. Furthermore, we concentrate on the exemplar-free setting~\cite{DBLP:conf/eccv/0002ZESZLRSPDP22,zhu2021prototype}, where no historical data can be fetched for rehearsal.

\section{Methodology}
Figure~\ref{fig:method}~(a) illustrates the architecture of our \mbox{InfLoRA} within a linear layer. Before learning the $t$-th new task, our \mbox{InfLoRA} expands a LoRA-like branch, which includes a dimensionality reduction matrix $\bm{B}_{t}\in\mathbb{R}^{r\times d_{I}}$ and a dimensionality increasing matrix $\bm{A}_{t}\in\mathbb{R}^{d_{O}\times r}$. 
Then, the forward propagation of this linear layer is modified as
\begin{align}\label{eq:taha}
  \bm{e}=&\bm{W}\bm{h}+\sum_{j=1}^{t}\bm{A}_{j}\bm{B}_{j}\bm{h}=\bm{W}_{t-1}\bm{h}+\bm{A}_{t}\bm{B}_{t}\bm{h}=\bm{W}_{t}\bm{h}.
\end{align}
Here, $\bm{W}_{t}=\bm{W}_{t-1}+\bm{A}_{t}\bm{B}_{t}=\bm{W}+\sum_{i=1}^{t}\bm{A}_{i}\bm{B}_{i}$. Similar to LoRA, our \mbox{InfLoRA} also initializes dimensionality increasing matrix $\bm{A}_{t}$ as $\bm{0}$. However, different from LoRA, which employs Gaussian distribution to initialize the dimensionality reduction matrix $\bm{B}$, our \mbox{InfLoRA} designs the dimensionality reduction matrix $\bm{B}_{t}$ before learning the $t$-th task. During the learning of the $t$-th task, \mbox{InfLoRA} fine-tunes $\bm{A}_{t}$ to learn the new task while keeping the pre-trained weight $\bm{W}$, all the old branches and the matrix $\bm{B}_{t}$ frozen. After learning the $t$-th tasks, for any given test sample belonging to the learned tasks, the model uses $\bm{W}_{t}$ and~(\ref{eq:taha}) to infer its label. This design ensures that our method is compatible with the class-incremental scenario where task identities are unknown during inference.


In the following subsections, we first build the relationship between our \mbox{InfLoRA} and the method that fine-tunes the pre-trained weight. Specifically, we show that fine-tuning parameters $\bm{A}_{t}$ is equivalent to fine-tuning the pre-trained weights $\bm{W}$ within a subspace spanned by the rows of $\bm{B}_{t}$. Note that $\bm{B}_{t}$ is designed before learning the $t$-th task, making this subspace pre-designed. Then, building upon this relationship, we introduce how our \mbox{InfLoRA} designs this subspace to eliminate the interference of the new task on the old tasks and make a good trade-off between stability and plasticity. 





\subsection{Relationship between \mbox{InfLoRA} and Fine-Tuning the Pre-Trained Weight}\label{sec:establish-relation}
When the $t$-th task arrives and our method has expanded a new branch, the forward propagation in this layer can be represented by~(\ref{eq:taha}). At this time, we can prove the following proposition:
\begin{proposition}\label{thm:relationship}
  When learning the $t$-th task with forward propagation represented by~(\ref{eq:taha}), fine-tuning $\bm{A}_{t}$ is equivalent to fine-tuning the pre-trained weight $\bm{W}$ within the subspace ${\rm span}\{\bm{b}_{1}^{t},...,\bm{b}_{r}^{t}\}$. 
Here, $\bm{b}_{i}^{t}$~($1\leq i\leq r$) denotes the $i$-th row vector of $\bm{B}_{t}$.
\end{proposition}
\begin{proof}
  When tuning the pre-trained weight $\bm{W}$ to learn the $t$-th task, we can compute the gradient of $\bm{W}$ based on the chain rule:
  \begin{align}\label{eq:full-grad}
    \frac{\partial \mathcal{L}}{\partial \bm{W}}=\frac{\partial \mathcal{L}}{\partial \bm{e}}\frac{\partial \bm{e}}{\partial \bm{W}}=\frac{\partial \mathcal{L}}{\partial \bm{e}}\bm{h}^{T}.
  \end{align}
  Here, $\mathcal{L}$ denotes the loss function. At this time, the change of $\bm{W}$ can be denoted as $\Delta \bm{W}=-\alpha\frac{\partial \mathcal{L}}{\partial \bm{W}}$, where $\alpha$ is the learning rate. Then, we can compute the change of the composed matrix $\bm{W}_{t}=\bm{W}+\sum_{j=1}^{t}\bm{A}_{j}\bm{B}_{j}$: 
  \begin{align}\label{eq:relationship-pre-trained}
    \Delta_{\bm{W}} \bm{W}_{t}=&[\bm{W}+\Delta\bm{W}+\sum_{j=1}^{t}\bm{A}_{j}\bm{B}_{j}]-(\bm{W}+\sum_{j=1}^{t}\bm{A}_{j}\bm{B}_{j})\nonumber\\
    =&\Delta \bm{W}=-\alpha\frac{\partial \mathcal{L}}{\partial \bm{W}_{t}}=-\alpha\frac{\partial \mathcal{L}}{\partial \bm{e}}\bm{h}^{T}
  \end{align}
  Here, we use $\Delta_{\bm{W}}\bm{W}_{t}$ to denote the change of the composed matrix $\bm{W}_{t}$ causing by the change of $\bm{W}$.

  Similarly, when tuning the expanded weight $\bm{A}_{t}$, we can get the gradient of $\bm{A}_{t}$ based on the chain rule:
  \begin{align}\label{eq:partial-grad}
    \frac{\partial \mathcal{L}}{\partial \bm{A}_{t}}=\frac{\partial \mathcal{L}}{\partial \bm{e}}\frac{\partial \bm{e}}{\partial \bm{A}_{t}}=\frac{\partial \mathcal{L}}{\partial \bm{e}}\bm{h}^{T}\bm{B}_{t}^{T}.
  \end{align}
  At this time, the change of $\bm{A}_{t}$ can be denoted as $\Delta \bm{A}_{t}=-\alpha\frac{\partial \mathcal{L}}{\partial \bm{A}_{t}}$. Then, we can compute the change of the composed matrix $\bm{W}_{t}=\bm{W}_{t-1}+\bm{A}_{t}\bm{B}_{t}$: 
  \begin{align}\label{eq:relationship-expanded-weight}
    \Delta_{\bm{A}_{t}} \bm{W}_{t}=&[\bm{W}_{t-1}+(\bm{A}_{t}+\Delta \bm{A}_{t})\bm{B}_{t}]-(\bm{W}_{t-1}+\bm{A}_{t}\bm{B}_{t})\nonumber\\
    =&\Delta \bm{A}_{t}\bm{B}_{t}=-\alpha\frac{\partial \mathcal{L}}{\partial \bm{A}_{t}}\bm{B}_{t}=-\alpha\frac{\partial \mathcal{L}}{\partial \bm{e}}\bm{h}^{T}\bm{B}_{t}^{T}\bm{B}_{t}\nonumber\\
    =&\Delta_{\bm{W}} \bm{W}_{t}\bm{B}_{t}^{T}\bm{B}_{t}
  \end{align}
  Here, we use $\Delta_{\bm{A}_{t}}\bm{W}_{t}$ to denote the change of the composed matrix $\bm{W}_{t}$ causing by the change of $\bm{A}_{t}$.
  The fourth equation in~(\ref{eq:relationship-expanded-weight}) holds because of~(\ref{eq:partial-grad}). The fifth equation in~(\ref{eq:relationship-expanded-weight}) holds because of~(\ref{eq:full-grad}).~(\ref{eq:relationship-expanded-weight}) shows that $\Delta_{\bm{A}_{t}}\bm{W}_{t}$ is equal to $\Delta_{\bm{W}}\bm{W}_{t}$ multiplying a projection matrix $\bm{B}_{t}^{T}\bm{B}_{t}$. Since $\bm{B}_{t}^{T}\bm{B}_{t}$ projects each row vector of $\Delta_{\bm{W}}\bm{W}_{t}$ into the subspace ${\rm span}\{\bm{b}_{1}^{t},...,\bm{b}_{r}^{t}\}$, Proposition~\ref{thm:relationship} holds.
\end{proof}
Proposition~\ref{thm:relationship} has demonstrated that using our \mbox{InfLoRA} to train the model is equivalent to directly fine-tuning the pre-trained weight $\bm{W}$ within the subspace ${\rm span}\{\bm{b}_{1}^{t},...,\bm{b}_{r}^{t}\}$. Therefore, before learning the $t$-th task, we can design matrix $\bm{B}_{t}$ such that learning the $t$-th task in the subspace ${\rm span}\{\bm{b}_{1}^{t},...,\bm{b}_{r}^{t}\}$ will not interfere with the performance of the model on the old tasks.

\subsection{Eliminating the Interference of the New Task on the Old Tasks}\label{sec:stability}
We first introduce the desired characteristics that \mbox{InfLoRA} aims to let the subspace ${\rm span}\{\bm{b}_{1}^{t},...,\bm{b}_{r}^{t}\}$ have. With these characteristics, \mbox{InfLoRA} can eliminate the interference of the new task on the old tasks and make a good trade-off between stability and plasticity. Then, we introduce how to design dimensionality reduction matrix $\bm{B}_{t}$ so that subspace ${\rm span}\{\bm{b}_{1}^{t},...,\bm{b}_{r}^{t}\}$ has these characteristics.
\subsubsection{Desired Characteristics}\label{sec:find}
First, \mbox{InfLoRA} aims to make the subspace ${\rm span}\{\bm{b}_{1}^{t},...,\bm{b}_{r}^{t}\}$ orthogonal to the gradients of all the old tasks. In this way, according to Proposition~\ref{thm:relationship}, the update of \mbox{InfLoRA}, which can be represented as $\Delta_{\bm{A}_{t}}\bm{W}_{t}$, will also be orthogonal to the gradient of the old tasks. Note that the idea of making the update for the new task orthogonal to the gradient of the old tasks to eliminate the interference of the new task on the old tasks has been proposed in many existing continual learning methods~\cite{DBLP:conf/iclr/SahaG021,lin2021trgp}. However, all these existing methods are designed for continual learning from scratch, involving updating all parameters of the model, which is incompatible with the setting in PEFT. On the contrary, our method is a PEFT method, which only tunes the parameters in $\bm{A}_{t}$.



Besides eliminating the interference of new tasks on old tasks, our \mbox{InfLoRA} further makes the subspace ${\rm span}\{\bm{b}_{1}^{t},...,\bm{b}_{r}^{t}\}$ lie in a subspace that the gradient of the new task lies in to make a good trade-off between stability and plasticity. Specifically, existing work~\cite{jie2023fact} has shown that during fine-tuning, the weight increments of pre-trained ViT exhibit redundancy in terms of weight rank. Therefore, the gradients of the new task lie in a low-dimensional subspace. Our method makes ${\rm span}\{\bm{b}_{1}^{t},...,\bm{b}_{r}^{t}\}$ not only orthogonal to the gradient of the old tasks but also lie in the subspace in which the gradients of the new task $t$ lie. By doing so, our method makes the model's focus on the new task when eliminating the interference of the new task on the old tasks, thereby making a good trade-off between stability and plasticity. Section~\ref{sec:ablation} verifies the effectiveness of these two characteristics.




\subsubsection{Designing Dimensionality Reduction Matrix}
InfLoRA first approximates the gradient space of the new task and old tasks. Here, we use $\mathcal{N}_{t}$ to represent the gradient space of the new task approximated by InfLoRA. Similarly, we use $\mathcal{M}_{t}$ to represent the gradient space of previous $t-1$ old tasks approximated by InfLoRA. We also use $\mathcal{M}_{t}^{\bot}$ to denote the residual gradient space, which is orthogonal to the space $\mathcal{M}_{t}$. Then, in order to meet the characteristics described in Section~\ref{sec:find}, \mbox{InfLoRA} ensures that each row of $\bm{B}_{t}$ lies in $\mathcal{N}_{t}\cap \mathcal{M}_{t}^{\bot}$. In other words, \mbox{InfLoRA} makes ${\rm span}\{\bm{b}_{1}^{t},...,\bm{b}_{r}^{t}\}\subseteq \mathcal{N}_{t}\cap \mathcal{M}_{t}^{\bot}$.

Existing works~\cite{DBLP:conf/iclr/SahaG021,liang2023adaptive} have shown that the gradient update of the linear layer lies in the span of the inputs. Please refer to supplementary material for a detailed explanation of this proposition. Therefore, InfLoRA uses the input matrix of the new task $t$ to approximate the gradient space of the new task. Specifically, \mbox{InfLoRA} computes the input matrix $\bm{H}_{t}=[\bm{h}_{1}^{t},...,\bm{h}_{n}^{t}]$, with each column of $\bm{H}_{t}$ representing an input vector of the $t$-th task. Then, InfLoRA considers $\mathcal{N}_{t}$ as the subspace spanned by the columns of matrix $\bm{H}_{t}$.

However, InfLoRA cannot use the input matrix of the old tasks to approximate the gradient space of the old tasks since the data from the old tasks is not available when the model learns the new tasks. Instead, existing methods such as gradient projection memory~(GPM)~\cite{DBLP:conf/iclr/SahaG021} and dual gradient projection memory~(DualGPM)~\cite{liang2023adaptive} can learn a matrix to preserve information about the gradients of the old tasks. \mbox{InfLoRA} incorporates DualGPM to preserve gradient information. With the assistance of DualGPM, the model can learn either a matrix $\bm{M}_{t}\in\mathbb{R}^{d_{I}\times k_{t}}$ or a matrix $\bm{M}_{t}^{\bot}\in\mathbb{R}^{d_{I}\times (d_{I}-k_{t})}$. Here, the columns of $\bm{M}_{t}$ contribute to the orthonormal bases of $\mathcal{M}_{t}$ and the columns of $\bm{M}_{t}^{\bot}$ contribute to the orthonormal bases of $\mathcal{M}_{t}^{\bot}$. $k_{t}$ denotes the dimension of $\mathcal{M}_{t}$. For detailed information of how DualGPM maintains orthonormal bases $\bm{M}_{t}$ or $\bm{M}_{t}^{\bot}$, please refer to supplementary material or the original paper~\cite{liang2023adaptive}.

After approximating the gradient space of the new task and old tasks, InfLoRA
gets the component of $\mathcal{N}_{t}$ which lies in $\mathcal{M}_{t}^{\bot}$. Specifically, when the model maintains $\mathcal{M}_{t}$, InfLoRA performs the operation
\begin{align}\label{eq:proj}
  \hat{\bm{H}}_{t}=\bm{H}_{t}-\bm{M}_{t}\bm{M}_{t}^{T}\bm{H}_{t}.
\end{align}
Similarly, when the model maintains $\mathcal{M}_{t}^{\bot}$, InfLoRA performs the operation
\begin{align}\label{eq:proj-dual}
  \hat{\bm{H}}_{t}=\bm{M}_{t}^{\bot}(\bm{M}_{t}^{\bot})^{T}\bm{H}_{t}.
\end{align}
Note that when $t=1$, $\mathcal{M}_{t}$ is a null space and $\hat{\bm{H}}_{t}=\bm{H}_{t}$.
Obviously, each column of $\hat{\bm{H}}_{t}$ lies in $\mathcal{N}_{t}\cap \mathcal{M}_{t}^{\bot}$. However, since $(\hat{\bm{H}}_{t})^{T}\in \mathbb{R}^{n\times d_{I}}$ and $\bm{B}_{t}\in \mathbb{R}^{r\times d_{I}}$ have different shapes, InfLoRA can not directly define $\bm{B}_{t}$ as $(\hat{\bm{H}}_{t})^{T}$. Note that $n\gg r$, InfLoRA uses the principal components of $(\hat{\bm{H}}_{t})^{T}$ to set $\bm{B}_{t}$. Specifically, singular value decomposotion~(SVD) is performed on $(\hat{\bm{H}}_{t})^{T}=\bm{V}_{t}\bm{\Sigma}_{t}\bm{U}_{t}$. Then, InfLoRA designs $\bm{B}_{t}$ by 
\begin{align}\label{eq:initialize-b}
  \bm{B}_{t}=(\bm{U}_{t})_{r}.
\end{align}
Here, $(\bm{U}_{t})_{r}$ denotes the rows of $\bm{U}_{t}$ corresponding to the top-$r$ singular values. Figure~\ref{fig:method}~(b) illustrates the pipeline of designing matrix $\bm{B}_{t}$.


Note that DualGPM expands subspace $\mathcal{M}_{t}$ and reduces subspace $\mathcal{M}_{t}^{\bot}$ when the number of tasks increases. Since InfLoRA constrains the update of the model within the subspace $\mathcal{N}_{t}\cap \mathcal{M}_{t}^{\bot}\subseteq \mathcal{M}_{t}^{\bot}$, the space for learning the new task reduces when the number of tasks increases. However, by adjusting the approximation error of the gradient for the old tasks, DualGPM can expand $\mathcal{M}_{t}$ slowly and reduce $\mathcal{M}_{t}^{\bot}$ slowly. Therefore, the constraints imposed by InfLoRA do not excessively affect the model's learning of new tasks. Please refer to supplementary material for a detailed explanation.

\begin{algorithm}[t]
  \caption{\mbox{InfLoRA} for Continual Learning}
  \label{alg:taha}
  \begin{algorithmic}[1]
  \small
  \STATE {\bfseries Input:} The data of different tasks $\{\mathcal{D}_{t}\}_{t=1}^{T}$, a pre-trained ViT model $f_{\bm{\Theta}}(\cdot)$.
  \STATE {\bfseries Output:} Network $f_{\bm{\Theta}}(\cdot)$ with learned parameters $\bm{W}_{t}$.
  \FOR {$t$ in $1:T$}
  \STATE {Design $\bm{B}_{t}$ through~(\ref{eq:initialize-b});}
  \STATE {Expand a new branch for the $t$-th task; }
  \FOR {$\mathcal{B}_{t}$ sampled from $\mathcal{D}_{t}$}
  \STATE {Compute the loss $\mathcal{L}(f_{\bm{\Theta}}(\mathcal{B}_{t}))$ through~(\ref{eq:masked-loss}) and update the parameters;}
  \ENDFOR
\STATE {Preserve the information about the gradient of the $t$-th task through DualGPM;}
  \ENDFOR
  \end{algorithmic}
\end{algorithm}


\subsection{Whole Process of \mbox{InfLoRA}}\label{sec:relation}
Algorithm~\ref{alg:taha} outlines the whole process of \mbox{InfLoRA} in continual learning. When the $t$-th new task arrives, \mbox{InfLoRA} first designs $\bm{B}_{t}$ through~(\ref{eq:initialize-b}) and expands a new branch. Then, \mbox{InfLoRA} learns the $t$-th task by fine-tuning the newly expanded branch. Please note that, based on empirical findings from existing methods~\cite{smith2023coda, gao2023unified}, we employ the local cross-entropy~(CE) loss as the learning objective, as it usually performs better than the global CE loss in continual learning methods based on PEFT. The local CE is the CE loss constrained to the classes of the current new task, which can be denoted as
\begin{align}\label{eq:masked-loss}
  \mathcal{L}(\mathcal{D}_{t})=\frac{1}{|\mathcal{D}_{t}|}\sum_{(\bm{x},y)\in \mathcal{D}_{t}}\mathcal{L}_{ce}({\rm mask}(h_{\bm{\Phi}}(f_{\bm{\Theta}}(\bm{x}))),y).
\end{align}
Here, ${\rm mask}(\cdot)$ is a function that filters out the logits of the old classes and $\mathcal{L}_{ce}$ denotes the standard CE loss. After learning the $t$-th new task, \mbox{InfLoRA} follows the DualGPM to preserve the information about the gradient of the $t$-th task.

Note that the branch corresponding to the $t$-th task will be frozen once the model has learned the $t$-th task. Since the expanded branches are linear transformations, we can integrate the old branches into the pre-trained weight to reduce the expanded parameters. Specifically, after learning the first task, \mbox{InfLoRA} integrates the first branch into the pre-trained weight and obtains the weight $\bm{W}_{1}=\bm{W}+\bm{A}_{1}\bm{B}_{1}$. Before learning the $t$-th new task~($t>1$), \mbox{InfLoRA} maintains the weight $\bm{W}_{t-1}$. After learning the $t$-th task, \mbox{InfLoRA} integrates the $t$-th branch into $\bm{W}_{t-1}$ and obtains $\bm{W}_{t}=\bm{W}_{t-1}+\bm{A}_{t}\bm{B}_{t}$. In this way, the parameters in $\bm{A}_{t}$ and $\bm{B}_{t}$ do not need to be maintained in the learning of subsequent tasks. Therefore, during the whole learning process, the number of parameters expanded by \mbox{InfLoRA} equals the number of parameters in a single branch. Since a single branch contains $(d_{I}+d_{O})r$ parameters, the number of parameters expanded by \mbox{InfLoRA} is $(d_{I}+d_{O})r$ all the time.


\begin{table*}[t]
  \caption{Results (\%) on ImageNet-R. Results are included for 5 tasks, 10 tasks, and 20 tasks. We report results averaged over 5 trials. %
  }
  \label{tab:imnet-r_main}
  \centering
  \setlength{\tabcolsep}{6pt}  
  
  \begin{tabular}{l||c c||c c||c c} 
  \hline 
  Tasks & \multicolumn{2}{c||}{5} & \multicolumn{2}{c||}{10} & \multicolumn{2}{c}{20} \\
  \hline
  \rule{0pt}{10pt} Method & $ACC_5$ ($\uparrow$) & $\overline{ACC}_5$ ($\uparrow$) & $ACC_{10}$ ($\uparrow$) & $\overline{ACC}_{10}$ ($\uparrow$) & $ACC_{20}$ ($\uparrow$) & $\overline{ACC}_{20}$ ($\uparrow$) \\
  \hline
  \emph{joint}  
  & $81.14 \pm 0.34$ & - 
  & $81.14 \pm 0.34$ & - 
  & $81.14 \pm 0.34$ & -   \\ 
  \hline
  \emph{sequential}   
  & $58.74 \pm 1.28$ & $72.91 \pm 0.28$       
  & $46.07 \pm 1.15$ & $62.91 \pm 0.68$ 
  & $34.62 \pm 0.85$ & $51.15 \pm 1.50$ \\
  L2P~\cite{wang2022learning} 
  & $64.13 \pm 0.78$ & $68.66 \pm 0.41$
  & $62.54 \pm 0.24$ & $67.98 \pm 0.27$ 
  & $57.92 \pm 0.28$ & $64.57 \pm 0.29$ \\
  DualPrompt~\cite{DBLP:conf/eccv/0002ZESZLRSPDP22} 
  & $67.88 \pm 0.17$ & $71.16 \pm 0.31$
  & $65.41 \pm 0.52$ & $69.39 \pm 0.43$ 
  & $61.00 \pm 0.72$ & $65.80 \pm 0.67$ \\
  CODA-P~\cite{smith2023coda}   
  & $73.09 \pm 0.21$ & $76.91 \pm 0.21$	
  & $71.47 \pm 0.35$ & $75.82 \pm 0.29$
  & $67.28 \pm 0.30$ & $72.34 \pm 0.17$ \\ 
  C-LoRA~\cite{DBLP:journals/corr/abs-2304-06027}
  & $75.85 \pm 0.31$ & $78.85 \pm 0.34$	
  & $71.89 \pm 0.45$ & $75.33 \pm 0.28$
  & $65.71 \pm 0.60$ & $70.63 \pm 0.85$ \\
  LAE~\cite{gao2023unified}
  & $73.84 \pm 0.14$ & $77.29 \pm 0.45$ 
  & $71.70 \pm 0.39$ & $76.71 \pm 0.10$ 
  & $66.98 \pm 0.35$ & $73.72 \pm 0.05$ \\
  \mbox{InfLoRA}-b5
  & $75.28 \pm 0.01$ & $78.95 \pm 0.08$
  & $74.13 \pm 0.18$ & $78.54 \pm 0.14$   
  & $68.41 \pm 0.29$ & $74.00 \pm 0.50$ \\
  \mbox{InfLoRA}
  & \textbf{77.52 $\pm$ 0.37} & \textbf{82.01 $\pm$ 0.12} 
  & \textbf{75.65 $\pm$ 0.14} & \textbf{80.82 $\pm$ 0.24} 
  & \textbf{71.01 $\pm$ 0.45} & \textbf{77.28 $\pm$ 0.45}  \\ 
  \hline
  \end{tabular}
  \vskip -0.1in
\end{table*}

\section{Experiments}
\subsection{Experimental Settings}\label{sec:exp-settings}
\myparagraph{Datasets and Evaluation Metric}
Similar to existing continual learning methods~\cite{gao2023unified,wang2022learning} based on PEFT, we use ImageNet-R~\cite{hendrycks2021many},  CIFAR100~\cite{krizhevsky2009learning}, and DomainNet~\cite{peng2019moment} to train and evaluate the models. Imagenet-R is generated through artistic processing of 200 classes from ImageNet~\cite{deng2009imagenet}. This dataset is introduced to continual learning by existing work~\cite{DBLP:conf/eccv/0002ZESZLRSPDP22} and has become a standard benchmark for continual learning methods based on PEFT. CIFAR100 is a dataset commonly used in existing continual learning works. DomainNet contains 345 classes and is introduced by some existing works~\cite{smith2023coda,wang2022s} for continual learning. Following existing continual learning work~\cite{smith2023coda}, we split ImageNet-R into 5, 10, and 20 tasks, with each task containing 40, 20, and 10 classes. We split CIFAR100 into 10 tasks, and each task constrains 10 classes. We split DomainNet into 5 tasks, and each task contains 69 classes.

Following existing continual learning methods~\cite{gao2023unified,wang2022learning}, we evaluate the performance of the model through two popular metrics, including the final accuracy $ACC_{T}$ and the averaged accuracy $\overline{ACC}_{T}=\frac{1}{T}\sum_{i=1}^{T}ACC_{i}$, where $T$ denotes the total number of tasks and $ACC_{i}$ is defined as
\begin{align}
  ACC_{i}=\frac{1}{i}\sum_{j=1}^{i}a_{i,j}.
\end{align}
Here, $a_{i,j}$ denotes the accuracy of the $j$-th task once the model has learned the $i$-th task.

\myparagraph{Baselines}
We compare our \mbox{InfLoRA} with state-of-the-art continual learning methods based on PEFT, including learn to prompt~(L2P)~\cite{wang2022learning}, DualPrompt~\cite{DBLP:conf/eccv/0002ZESZLRSPDP22}, continual decomposed attention-based prompt~(CODA-P)~\cite{smith2023coda}, learning accumulation ensemble~(LAE)~\cite{gao2023unified}, continual low-rank adaptation~(C-LoRA)~\cite{DBLP:journals/corr/abs-2304-06027}.
For LAE, we implement it with LoRA~\cite{hu2021lora}. 
Following existing works~\cite{smith2023coda,gao2023unified}, we also include two methods without continual learning, \emph{joint} and \emph{sequential}, in the comparison. Here, \emph{joint} denotes the method that learns all the tasks jointly, while \emph{sequential} denotes the method that learns all the tasks sequentially without any operation to overcome the forgetting of the model. The accuracy of \emph{joint} can be treated as the accuracy upper-bound and the accuracy of \emph{sequential} can be treated as the accuracy lower-bound.

\myparagraph{Architecture and Training Details} 
We follow existing works~\cite{gao2023unified,DBLP:conf/eccv/0002ZESZLRSPDP22} to perform experiments. Specifically, we use the ViT-B/16 backbone~\cite{dosovitskiy2020image} supervised pre-trained on ImageNet 21K as the pre-trained model.


For all the methods, we follow existing works~\cite{smith2023coda,wang2022learning,gao2023unified} and use the Adam~\cite{kingma2014adam} optimizer with running averages of gradient and its square~($\beta_{1} = 0.9$, $\beta_{2} = 0.999$). Each task is trained for 50 epochs on ImageNet-R, 20 epochs on CIFAR100 and 5 epochs on DomainNet. The batch size is set to 128 for all the experiments. Since our \mbox{InfLoRA} shares a similar architecture to LoRA, we follow existing work~\cite{gao2023unified} and insert the architecture of our \mbox{InfLoRA} in the key and value of the attention module. Furthermore, existing method DualPrompt~\cite{DBLP:conf/eccv/0002ZESZLRSPDP22} treats the inserted blocks as hyperparameters and searches for the best positions for their prompts. On the contrary, we insert the architecture of \mbox{InfLoRA} for all the Transformer blocks to avoid searching. We also implement a variant of our method, which inserts the bottom 5 Transformer blocks like existing methods DualPrompt and CODA-P. We call this variant \mbox{InfLoRA}-b5. 
As for the hyperparameter $r$, we determine its value through a grid search on a validation dataset. 

\begin{table*}[t]
  \caption{Results (\%) on CIFAR100 and DomainNet. We report results over 5 trials. %
  }
  \label{tab:cifar100-domainnet}
  \centering
  \setlength{\tabcolsep}{10pt}  
  
  \begin{tabular}{l||c c||c c} 
  \hline 
  Tasks & \multicolumn{2}{c||}{CIFAR100} & \multicolumn{2}{c}{DomainNet}  \\
  \hline
  \rule{0pt}{10pt} Method & $ACC_{10}$ ($\uparrow$) & $\overline{ACC}_{10}$ ($\uparrow$) & $ACC_{5}$ ($\uparrow$) & $\overline{ACC}_{5}$ ($\uparrow$) \\
  \hline
  \emph{joint}  
  & $91.92 \pm 0.05$ & - 
  & $77.72 \pm 0.04$ & - \\ 
  \hline
  \emph{sequential}   
  & $62.18 \pm 3.59$ & $80.42 \pm 0.23$       
  & $53.44 \pm 1.21$ & $69.09 \pm 0.33$ \\
  L2P~\cite{wang2022learning} 
  & $82.48 \pm 0.20$ & $87.64 \pm 0.25$
  & $70.16 \pm 0.05$ & $75.60 \pm 0.03$ \\
  DualPrompt~\cite{DBLP:conf/eccv/0002ZESZLRSPDP22} 
  & $84.42 \pm 0.30$ & $90.06 \pm 0.07$
  & $72.14 \pm 0.05$ & $77.71 \pm 0.06$ \\
  CODA-P~\cite{smith2023coda}   
  & $86.62 \pm 0.11$ & $91.08 \pm 0.28$	
  & $73.23 \pm 0.13$ & $78.72 \pm 0.07$ \\ 
  C-LoRA~\cite{DBLP:journals/corr/abs-2304-06027}
  & $82.97 \pm 0.47$ & $88.81 \pm 0.34$ 
  & $69.34 \pm 0.13$ & $75.25 \pm 0.11$ \\
  LAE~\cite{gao2023unified}
  & $84.15 \pm 0.10$ & $89.84 \pm 0.03$ 
  & $66.85 \pm 0.40$ & $75.01 \pm 0.17$ \\
  \mbox{InfLoRA}-b5
  & \textbf{87.06 $\pm$ 0.25} & $91.59 \pm 0.13$
  & $73.26 \pm 0.50$ & $78.82 \pm 0.34$ \\
  \mbox{InfLoRA}
  & $86.51 \pm 0.73$ & \textbf{91.70 $\pm$ 0.32} 
  & \textbf{74.53 $\pm$ 0.23} & \textbf{79.57 $\pm$ 0.57} \\ 
  \hline
  \end{tabular}
  \vskip -0.1in
\end{table*}

\begin{figure*}[t]
  \setlength{\belowcaptionskip}{-0.5cm}
  \centering
  \includegraphics[width=0.95\textwidth]{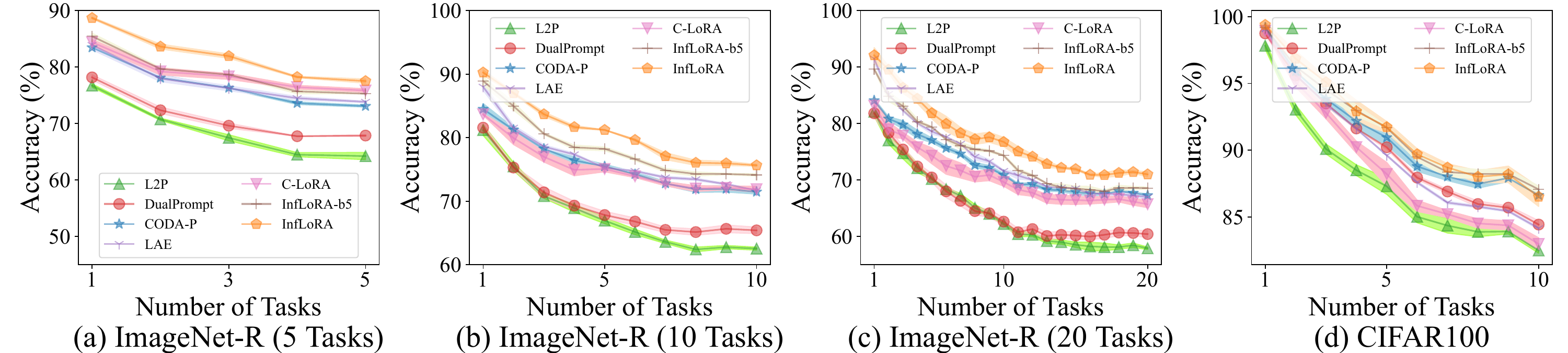}
 \captionsetup{skip=2pt}
 \caption{Variation of the performance of different methods during the learning of ImageNet-R and CIFAR100.}
  \label{fig:change-imagenet-r}
\end{figure*}

\subsection{Experimental Results}
\myparagraph{Accuracy}
Table~\ref{tab:imnet-r_main} shows the results of different methods on ImageNet-R with a different number of tasks. Table~\ref{tab:cifar100-domainnet} shows the results of different methods on CIFAR100 and DomainNet. We can find that our methods \mbox{InfLoRA} and \mbox{InfLoRA}-b5 outperform existing continual learning methods.

Figure~\ref{fig:change-imagenet-r} shows the variation of the accuracy of different continual learning methods on ImageNet-R and CIFAR100. We can find that our method outperforms existing methods not only at the end of the learning but also throughout the whole learning process. This indicates that our InfLoRA eliminates the interference of the new task on the old tasks and thus the accuracy of our method decreases slower compared to other methods.
\myparagraph{Analysis of Expanded Parameters}\label{sec:expanded}
Figure~\ref{fig:hyper} shows the number of expanded parameters and the accuracy of different methods on ImageNet-R and CIFAR100. For L2P, DualPrompt and CODA-P, their expanded parameters are included in the added prompts and corresponding key. For LAE, its expanded parameters are the inserted LoRA modules and an additional copy. For C-LoRA, its expanded parameters are inserted LoRA modules. For our method, the expanded parameters are $\bm{B}_{t}$ and $\bm{A}_{t}$. The details of computing the number of expanded parameters for different methods are given in supplementary material. We can find that CODA-P and C-LoRA expand much more parameters than other methods. Furthermore, our methods \mbox{InfLoRA} and \mbox{InfLoRA}-b5 expand comparable parameters to L2P, DualPrompt and LAE but perform better than these methods.

\begin{table*}[t]
  \caption{Results of different variants on ImageNet-R with a different number of tasks.  %
  }
  \label{tab:ab-imnet-r_main}
  \centering
  \setlength{\tabcolsep}{3.5pt}  
  
  \begin{tabular}{c||c c||c c||c c} 
  \hline 
  Tasks & \multicolumn{2}{c||}{5} & \multicolumn{2}{c||}{10} & \multicolumn{2}{c}{20} \\
  \hline
  \rule{0pt}{10pt} & $ACC_5$ ($\uparrow$) & $\overline{ACC}_5$ ($\uparrow$) & $ACC_{10}$ ($\uparrow$) & $\overline{ACC}_{10}$ ($\uparrow$) & $ACC_{20}$ ($\uparrow$) & $\overline{ACC}_{20}$ ($\uparrow$) \\
  \hline
  Random $\rightarrow\bm{B}_{t}$ 
  & $72.49 \pm 0.38$ & $79.40 \pm 0.29$
  & $67.38 \pm 0.41$ & $76.62 \pm 0.06$
  & $56.17 \pm 0.29$ & $69.24 \pm 0.35$ \\
  $\mathcal{N}_{t}\rightarrow \bm{B}_{t}$ 
  & $67.01 \pm 0.11$ & $76.09 \pm 0.04$
  & $57.91 \pm 0.30$ & $70.23 \pm 0.59$ 
  & $40.73 \pm 0.29$ & $59.68 \pm 0.52$ \\
  $\mathcal{M}_{t}^{\bot}\rightarrow \bm{B}_{t}$   
  & $75.94 \pm 0.53$ & $80.69 \pm 0.27$	
  & $74.61 \pm 0.62$ & $79.67 \pm 0.27$
  & $68.79 \pm 0.42$ & $75.74 \pm 0.26$  \\ 
  $\mathcal{N}_{t}\cap \mathcal{M}_{t}^{\bot}\rightarrow \bm{B}_{t}$~(\mbox{InfLoRA})
  & \textbf{77.52 $\pm$ 0.37} & \textbf{82.01 $\pm$ 0.12} 
  & \textbf{75.65 $\pm$ 0.14} & \textbf{80.82 $\pm$ 0.24} 
  & \textbf{71.01 $\pm$ 0.45} & \textbf{77.28 $\pm$ 0.45}  \\
  \hline
  \end{tabular}
  \vskip -0.15in
\end{table*}

\myparagraph{Ablation Study}\label{sec:ablation}
We perform experiment to verify the effectiveness of designing dimensionality reduction matrix $\bm{B}_{t}$ by~(\ref{eq:initialize-b}). Specifically, we explore three different variants for designing $\bm{B}_{t}$. The first variant designs $\bm{B}_{t}$ randomly using Gaussian distribution. We call this variant `Random $\rightarrow\bm{B}_{t}$'. The second variant discards the operation in~(\ref{eq:proj}) or~(\ref{eq:proj-dual}) and directly sets $\hat{\bm{H}}_{t}=\bm{H}_{t}$. Through this way, this variant ensures that each row of $\bm{B}_{t}$ lies in $\mathcal{N}_{t}$ while ignoring $\mathcal{M}_{t}^{\bot}$. We call this variant `$\mathcal{N}_{t}\rightarrow\bm{B}_{t}$'. The third variant does not compute the input matrix but initializes $\bm{H}_{t}$ using a Gaussian distribution before applying the operation in~(\ref{eq:proj}) or~(\ref{eq:proj-dual}). In this way, this variant ensures that each row of $\bm{B}_{t}$ lies in ${\mathcal{M}}_{t}^{\bot}$ while ignoring $\mathcal{N}_{t}$. We call this variant `${\mathcal{M}}_{t}^{\bot}\rightarrow\bm{B}_{t}$'. Since our method focuses both ${\mathcal{M}}_{t}^{\bot}$ and $\mathcal{N}_{t}$, we use $\mathcal{N}_{t}\cap {\mathcal{M}}_{t}^{\bot}\rightarrow \bm{B}_{t}$ to represent our method.

\begin{figure}[t]
  \centering
  \includegraphics[width=\columnwidth]{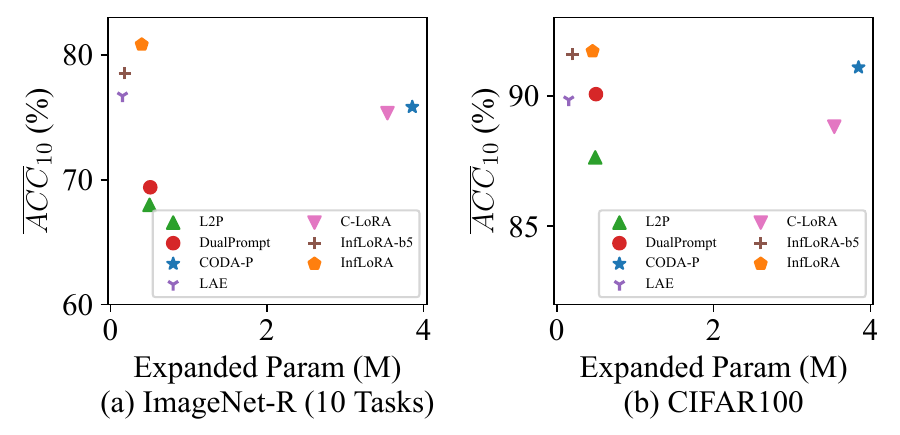}
 \captionsetup{skip=2pt}
 \caption{Variation of the performance of different methods during the learning of ImageNet-R and CIFAR100.}
  \label{fig:hyper}
  \vskip -0.2in
\end{figure}

Table~\ref{tab:ab-imnet-r_main} shows the results of our method and its variants. We can find that all these variants fail to perform as well as our method.
To further demonstrate the performance of different variants, we show the relative accuracy of different tasks after the model learns them all in Figure~\ref{fig:old_new}. Here, relative accuracy is the accuracy of different variants
minus the accuracy of our InfLoRA. Note that the last task is the new task, and the other tasks are old tasks in Figure~\ref{fig:old_new}. As we can see, `Random $\rightarrow\bm{B}_{t}$' and `$\mathcal{N}_{t}\rightarrow\bm{B}_{t}$' outperform `${\mathcal{M}}_{t}^{\bot}\rightarrow\bm{B}_{t}$' on the new task but shows much lower accuracy than `${\mathcal{M}}_{t}^{\bot}\rightarrow\bm{B}_{t}$' and our InfLoRA on the old tasks. This means these two variants fail to eliminate the inference of the new task on the old tasks, making the model suffer from low stability. On the contrary, `${\mathcal{M}}_{t}^{\bot}\rightarrow\bm{B}_{t}$' shows the lowest performance on the new task. This means `${\mathcal{M}}_{t}^{\bot}\rightarrow\bm{B}_{t}$' ignores the plasticity of the model. Our method outperforms all the variants on most of the tasks. This shows that our method can eliminate the interference of the new task on the old tasks and make a better trade-off between stability and plasticity than these variants.

\myparagraph{Varying the Pre-Trained Model}
We also follow the existing method~\cite{wang2023hierarchical} and
perform experiments using a ViT-B/16 pre-trained with two different self-supervised methods, including DINO~\cite{caron2021emerging} and iBOT~\cite{zhou2021image}. All experimental settings, except for the choice of the pre-trained model, are kept consistent with the details outlined in Section~\ref{sec:exp-settings}.

\begin{figure}[t]
  \centering
  \includegraphics[width=\columnwidth]{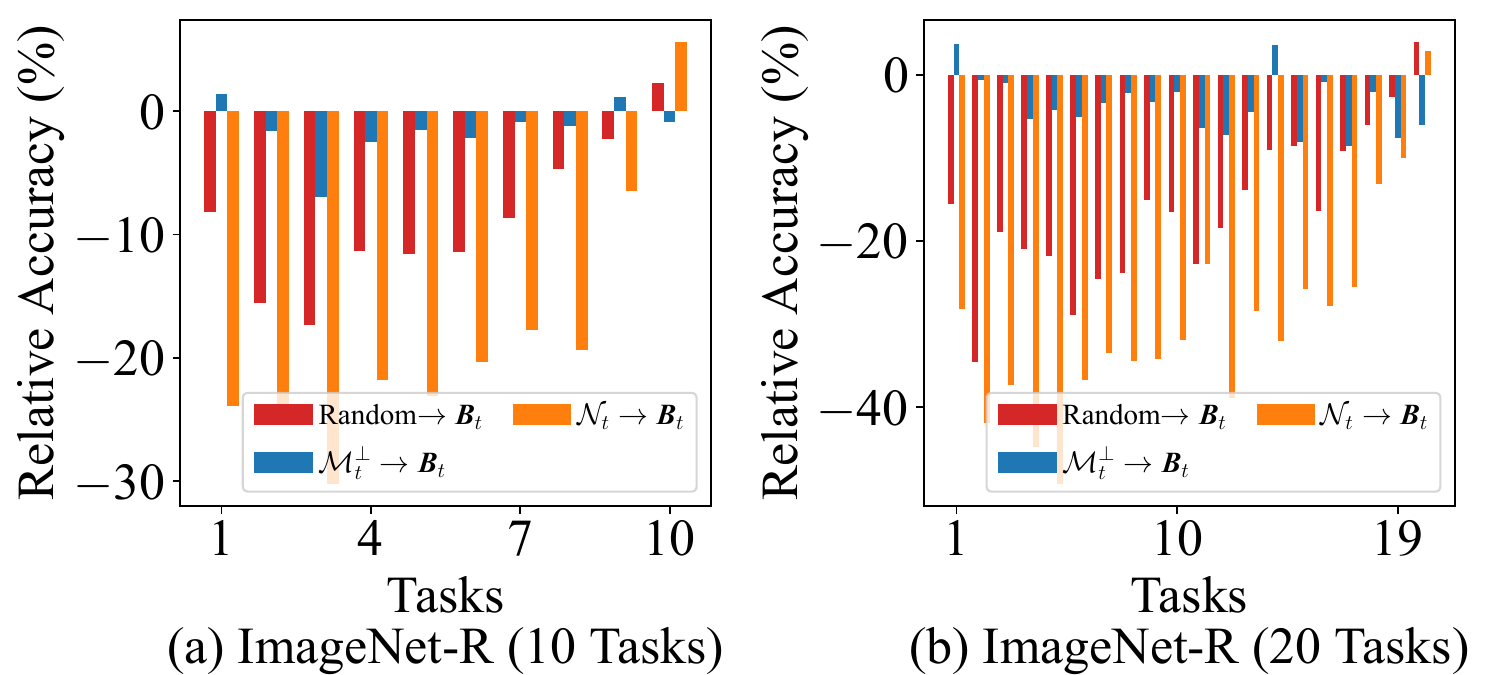}
 \captionsetup{skip=2pt}
 \caption{Relative accuracy of different tasks. Relative accuracy is the accuracy of different variants
  minus the accuracy of InfLoRA.}
  \label{fig:old_new}
  \vskip -0.1in
\end{figure}

\begin{table}[t]
  \caption{Results (\%) of different methods on ImageNet-R~(10 tasks) using various self-supervised pre-trained models. Here, DINO-1k and iBOT-1k indicate that the ViT is pre-trained on ImageNet-1k using these respective methods. %
  }
  \label{tab:cifar100_vary_pre-trained}
  \centering
  \setlength{\tabcolsep}{3pt}  
  
  \begin{tabular}{c|l||c c} 
  \hline
  \rule{0pt}{10pt} & Method & $ACC_{10}$ ($\uparrow$) & $\overline{ACC}_{10}$ ($\uparrow$) \\
  \hline
  \multirow{6}{*}{DINO-1k} & L2P~\cite{wang2022learning} 
  & $56.71 \pm 0.12$ & $63.59 \pm 0.21$ \\
  & DualPrompt~\cite{DBLP:conf/eccv/0002ZESZLRSPDP22} 
  & $60.23 \pm 0.42$ & $66.57 \pm 0.25$ \\
  & CODA-P~\cite{smith2023coda}   
  & $64.02 \pm 0.68$ & $71.50 \pm 0.42$	\\ 
  & C-LoRA~\cite{DBLP:journals/corr/abs-2304-06027}
  & $63.07 \pm 0.36$ & $68.09 \pm 0.41$ \\
  & LAE~\cite{gao2023unified}
  & $61.03 \pm 0.27$ & $69.89 \pm 0.15$ \\
  & \mbox{InfLoRA}-b5
  & $66.16 \pm 0.14$ & $73.01 \pm 0.17$ \\
  & \mbox{InfLoRA}
  & \textbf{68.31 $\pm$ 0.28} & \textbf{76.15 $\pm$ 0.05}  \\ 
  \hline
  \multirow{6}{*}{iBOT-1k} & L2P~\cite{wang2022learning} 
  & $60.80 \pm 0.35$ & $66.58 \pm 0.28$ \\
  & DualPrompt~\cite{DBLP:conf/eccv/0002ZESZLRSPDP22} 
  & $63.78 \pm 0.38$ & $68.88 \pm 0.16$ \\
  & CODA-P~\cite{smith2023coda}   
  & $68.02 \pm 0.48$ & $74.28 \pm 0.47$ \\ 
  & C-LoRA~\cite{DBLP:journals/corr/abs-2304-06027}
  & $68.60 \pm 0.07$ & $73.47 \pm 0.28$ \\
  & LAE~\cite{gao2023unified} 
  & $64.14 \pm 0.29$ & $72.59 \pm 0.22$ \\
  & \mbox{InfLoRA}-b5
  & $69.72 \pm 0.44$ & $76.11 \pm 0.13$ \\
  & \mbox{InfLoRA}
  & \textbf{71.84 $\pm$ 0.09} & \textbf{78.29 $\pm$ 0.09} \\
  \hline
  \end{tabular}
  \vskip -0.2in
\end{table}


Table~\ref{tab:cifar100_vary_pre-trained} shows the results of different methods on ImageNet-R when using various pre-trained models. Comparing these results to those in Table~\ref{tab:imnet-r_main}, we can find that the performance of all methods utilizing self-supervised pre-trained models is lower than the performance of the corresponding methods using supervised pre-trained models. However, our methods still outperform all other methods.

\myparagraph{Combining with Classifier Alignment}
Slow learner with classifier alignment~(SLCA)~\cite{DBLP:conf/iccv/ZhangWKCW23} utilizes feature statistics to align classifiers, demonstrating superior performance compared to methods without aligned classifiers. Our InfLoRA can be combined with classifier alignment~(CA) to get better performance. Specifically, after learning the $t$-th task with parameters $\bm{A}_{t}$ and $\bm{B}_{t}$ and loss~(\ref{eq:masked-loss}), we collect features $\bm{F}_{t}=\{\bm{r}_{i,t}\}_{i=1}^{n_{t}}$ of the $t$-th task. Here, $\bm{r}_{i,t}=f(\bm{x}_{i,t})$ denotes the features extracted by backbone $f_{\bm{\Theta}}(\cdot)$. Then, mean and covariance of features for each class are computed and saved. After that, for each class $c$ the model has seen during continual learning, $S$ samples are sampled from Gaussian distribution $\mathcal{N}(\bm{\mu}_{c},\bm{\Sigma}_{c})$. Here, $\bm{\mu}_{c}$ and covariance $\bm{\Sigma}_{c}$ denote the mean and covariance of the class $c$. Finally, we align the classifier using standard cross-entropy and these samples.
The details of this experiment are given in supplementary material.


Table~\ref{tab:combine-with-ca} shows that our method InfLoRA+CA outperforms SLCA. Note that SLCA tunes all the parameters of the model while our method InfLoRA only tunes the parameters in $\bm{A}_{t}$. Therefore, our InfLoRA+CA is much more efficient than SLCA.

\begin{table}[t]
  \caption{Results (\%) of different methods on ImageNet-R~(10 tasks) and CIFAR100 using classifier alignment~(CA) technique. %
  }
  \label{tab:combine-with-ca}
  \centering
  \setlength{\tabcolsep}{3pt}  
  
  \begin{tabular}{c|l||c c} 
  \hline
  \rule{0pt}{10pt} & Method & $ACC_{10}$ ($\uparrow$) & $\overline{ACC}_{10}$ ($\uparrow$) \\
  \hline
  \multirow{2}{*}{CIFAR100} & SLCA~\cite{DBLP:conf/iccv/ZhangWKCW23}
  & $91.06 \pm 0.24$ & $93.65 \pm 0.19$ \\
  & \mbox{InfLoRA}+CA
  & \textbf{91.59 $\pm$ 0.08} & \textbf{94.39 $\pm$ 0.05}  \\
  \hline
  \multirow{2}{*}{ImageNet-R} & SLCA~\cite{DBLP:conf/iccv/ZhangWKCW23} 
  & $77.34 \pm 0.25$ & $81.35 \pm 0.16$ \\
  & \mbox{InfLoRA+CA}
  & \textbf{79.78 $\pm$ 0.25} & \textbf{83.38 $\pm$ 0.19} \\
  \hline
  \end{tabular}
  \vskip -0.15in
\end{table}

\section{Conclusion}
In this work, we propose a new method, called interference-free low-rank adaptation~(\mbox{InfLoRA}), for continual learning. \mbox{InfLoRA} injects a small number of parameters to reparameterize the pre-trained weights and shows that fine-tuning these injected parameters is equivalent to fine-tuning the pre-trained weights within a subspace. Furthermore, \mbox{InfLoRA} designs this subspace to eliminate the interference of the new task on the old tasks, making a good trade-off between stability and plasticity. Experimental results show that \mbox{InfLoRA} outperforms existing state-of-the-art continual learning methods on multiple datasets.


\section*{Acknowledgment}
This work is supported by NSFC~(No.62192783), National Key R\&D Program of China~(No.2020YFA0713901), and Fundamental Research Funds for the Central Universities~(No.020214380108).

{
    \small
    \bibliographystyle{ieeenat_fullname}
    \bibliography{main}

\begin{thebibliography}{51}
\providecommand{\natexlab}[1]{#1}
\providecommand{\url}[1]{\texttt{#1}}
\expandafter\ifx\csname urlstyle\endcsname\relax
  \providecommand{\doi}[1]{doi: #1}\else
  \providecommand{\doi}{doi: \begingroup \urlstyle{rm}\Url}\fi

\bibitem[Aljundi et~al.(2018)Aljundi, Babiloni, Elhoseiny, Rohrbach, and
  Tuytelaars]{aljundi2018memory}
Rahaf Aljundi, Francesca Babiloni, Mohamed Elhoseiny, Marcus Rohrbach, and
  Tinne Tuytelaars.
\newblock Memory aware synapses: Learning what (not) to forget.
\newblock In \emph{Proceedings of the European Conference on Computer Vision},
  pages 139--154, 2018.

\bibitem[Aljundi et~al.(2019{\natexlab{a}})Aljundi, Belilovsky, Tuytelaars,
  Charlin, Caccia, Lin, and Page{-}Caccia]{DBLP:conf/nips/AljundiBTCCLP19}
Rahaf Aljundi, Eugene Belilovsky, Tinne Tuytelaars, Laurent Charlin, Massimo
  Caccia, Min Lin, and Lucas Page{-}Caccia.
\newblock Online continual learning with maximal interfered retrieval.
\newblock In \emph{Advances in Neural Information Processing Systems}, pages
  11849--11860, 2019{\natexlab{a}}.

\bibitem[Aljundi et~al.(2019{\natexlab{b}})Aljundi, Lin, Goujaud, and
  Bengio]{DBLP:conf/nips/AljundiLGB19}
Rahaf Aljundi, Min Lin, Baptiste Goujaud, and Yoshua Bengio.
\newblock Gradient based sample selection for online continual learning.
\newblock In \emph{Advances in Neural Information Processing Systems}, pages
  11816--11825, 2019{\natexlab{b}}.

\bibitem[Boschini et~al.(2022)Boschini, Bonicelli, Porrello, Bellitto, Pennisi,
  Palazzo, Spampinato, and Calderara]{boschini2022transfer}
Matteo Boschini, Lorenzo Bonicelli, Angelo Porrello, Giovanni Bellitto, Matteo
  Pennisi, Simone Palazzo, Concetto Spampinato, and Simone Calderara.
\newblock Transfer without forgetting.
\newblock In \emph{Proceedings of the European Conference on Computer Vision},
  pages 692--709, 2022.

\bibitem[Caron et~al.(2021)Caron, Touvron, Misra, J{\'e}gou, Mairal,
  Bojanowski, and Joulin]{caron2021emerging}
Mathilde Caron, Hugo Touvron, Ishan Misra, Herv{\'e} J{\'e}gou, Julien Mairal,
  Piotr Bojanowski, and Armand Joulin.
\newblock Emerging properties in self-supervised vision transformers.
\newblock In \emph{Proceedings of the IEEE/CVF International Conference on
  Computer Vision}, pages 9650--9660, 2021.

\bibitem[Chen et~al.(2022)Chen, Ge, Tong, Wang, Song, Wang, and
  Luo]{chen2022adaptformer}
Shoufa Chen, Chongjian Ge, Zhan Tong, Jiangliu Wang, Yibing Song, Jue Wang, and
  Ping Luo.
\newblock Adaptformer: Adapting vision transformers for scalable visual
  recognition.
\newblock \emph{Advances in Neural Information Processing Systems}, pages
  16664--16678, 2022.

\bibitem[Chrysakis and Moens(2020)]{chrysakis2020online}
Aristotelis Chrysakis and Marie-Francine Moens.
\newblock Online continual learning from imbalanced data.
\newblock In \emph{Proceedings of the International Conference on Machine
  Learning}, pages 1952--1961, 2020.

\bibitem[Deng et~al.(2009)Deng, Dong, Socher, Li, Li, and
  Fei-Fei]{deng2009imagenet}
Jia Deng, Wei Dong, Richard Socher, Li-Jia Li, Kai Li, and Li Fei-Fei.
\newblock Imagenet: A large-scale hierarchical image database.
\newblock In \emph{Proceedings of the IEEE/CVF Conference on Computer Vision
  and Pattern Recognition}, pages 248--255, 2009.

\bibitem[Devlin et~al.(2019)Devlin, Chang, Lee, and
  Toutanova]{DBLP:conf/naacl/DevlinCLT19}
Jacob Devlin, Ming{-}Wei Chang, Kenton Lee, and Kristina Toutanova.
\newblock {BERT:} pre-training of deep bidirectional transformers for language
  understanding.
\newblock In \emph{Proceedings of the Conference of the North American Chapter
  of the Association for Computational Linguistics}, pages 4171--4186, 2019.

\bibitem[Dosovitskiy et~al.(2021)Dosovitskiy, Beyer, Kolesnikov, Weissenborn,
  Zhai, Unterthiner, Dehghani, Minderer, Heigold, Gelly,
  et~al.]{dosovitskiy2020image}
Alexey Dosovitskiy, Lucas Beyer, Alexander Kolesnikov, Dirk Weissenborn,
  Xiaohua Zhai, Thomas Unterthiner, Mostafa Dehghani, Matthias Minderer, Georg
  Heigold, Sylvain Gelly, et~al.
\newblock An image is worth 16x16 words: Transformers for image recognition at
  scale.
\newblock In \emph{International Conference on Learning Representations}, 2021.

\bibitem[Fu et~al.(2022)Fu, Chen, Lee, and Lee]{fu2022adapterbias}
Chin-Lun Fu, Zih-Ching Chen, Yun-Ru Lee, and Hung-Yi Lee.
\newblock Adapterbias: Parameter-efficient token-dependent representation shift
  for adapters in nlp tasks.
\newblock In \emph{Findings of the Association for Computational Linguistics},
  pages 2608--2621, 2022.

\bibitem[Gao et~al.(2023)Gao, Zhao, Sun, Xi, Zhang, Ghanem, and
  Zhang]{gao2023unified}
Qiankun Gao, Chen Zhao, Yifan Sun, Teng Xi, Gang Zhang, Bernard Ghanem, and
  Jian Zhang.
\newblock A unified continual learning framework with general
  parameter-efficient tuning.
\newblock In \emph{Proceedings of the IEEE/CVF International Conference on
  Computer Vision}, pages 11449--11459, 2023.

\bibitem[He et~al.(2022)He, Chen, Xie, Li, Doll{\'a}r, and
  Girshick]{he2022masked}
Kaiming He, Xinlei Chen, Saining Xie, Yanghao Li, Piotr Doll{\'a}r, and Ross
  Girshick.
\newblock Masked autoencoders are scalable vision learners.
\newblock In \emph{Proceedings of the IEEE/CVF Conference on Computer Vision
  and Pattern Recognition}, pages 16000--16009, 2022.

\bibitem[Hendrycks et~al.(2021)Hendrycks, Basart, Mu, Kadavath, Wang, Dorundo,
  Desai, Zhu, Parajuli, Guo, et~al.]{hendrycks2021many}
Dan Hendrycks, Steven Basart, Norman Mu, Saurav Kadavath, Frank Wang, Evan
  Dorundo, Rahul Desai, Tyler Zhu, Samyak Parajuli, Mike Guo, et~al.
\newblock The many faces of robustness: A critical analysis of
  out-of-distribution generalization.
\newblock In \emph{Proceedings of the IEEE/CVF International Conference on
  Computer Vision}, pages 8340--8349, 2021.

\bibitem[Houlsby et~al.(2019)Houlsby, Giurgiu, Jastrzebski, Morrone,
  De~Laroussilhe, Gesmundo, Attariyan, and Gelly]{houlsby2019parameter}
Neil Houlsby, Andrei Giurgiu, Stanislaw Jastrzebski, Bruna Morrone, Quentin
  De~Laroussilhe, Andrea Gesmundo, Mona Attariyan, and Sylvain Gelly.
\newblock Parameter-efficient transfer learning for nlp.
\newblock In \emph{Proceedings of the International Conference on Machine
  Learning}, pages 2790--2799, 2019.

\bibitem[Hu et~al.(2022)Hu, Wallis, Allen-Zhu, Li, Wang, Wang, Chen,
  et~al.]{hu2021lora}
Edward~J Hu, Phillip Wallis, Zeyuan Allen-Zhu, Yuanzhi Li, Shean Wang, Lu Wang,
  Weizhu Chen, et~al.
\newblock Lora: Low-rank adaptation of large language models.
\newblock In \emph{International Conference on Learning Representations}, 2022.

\bibitem[Hung et~al.(2019)Hung, Tu, Wu, Chen, Chan, and
  Chen]{DBLP:conf/nips/Hung0WCCC19}
Steven C.~Y. Hung, Cheng{-}Hao Tu, Cheng{-}En Wu, Chien{-}Hung Chen, Yi{-}Ming
  Chan, and Chu{-}Song Chen.
\newblock Compacting, picking and growing for unforgetting continual learning.
\newblock In \emph{Advances in Neural Information Processing Systems}, pages
  13647--13657, 2019.

\bibitem[Jia et~al.(2022)Jia, Tang, Chen, Cardie, Belongie, Hariharan, and
  Lim]{DBLP:conf/eccv/JiaTCCBHL22}
Menglin Jia, Luming Tang, Bor{-}Chun Chen, Claire Cardie, Serge~J. Belongie,
  Bharath Hariharan, and Ser{-}Nam Lim.
\newblock Visual prompt tuning.
\newblock In \emph{Proceedings of the European Conference on Computer Vision},
  pages 709--727, 2022.

\bibitem[Jie and Deng(2023)]{jie2023fact}
Shibo Jie and Zhi-Hong Deng.
\newblock Fact: Factor-tuning for lightweight adaptation on vision transformer.
\newblock In \emph{Proceedings of the AAAI Conference on Artificial
  Intelligence}, pages 1060--1068, 2023.

\bibitem[Jung et~al.(2020)Jung, Ahn, Cha, and Moon]{jung2020continual}
Sangwon Jung, Hongjoon Ahn, Sungmin Cha, and Taesup Moon.
\newblock Continual learning with node-importance based adaptive group sparse
  regularization.
\newblock \emph{Advances in Neural Information Processing Systems}, pages
  3647--3658, 2020.

\bibitem[Khan et~al.(2023)Khan, Naeem, Van~Gool, Stricker, Tombari, and
  Afzal]{khan2023introducing}
Muhammad Gul Zain~Ali Khan, Muhammad~Ferjad Naeem, Luc Van~Gool, Didier
  Stricker, Federico Tombari, and Muhammad~Zeshan Afzal.
\newblock Introducing language guidance in prompt-based continual learning.
\newblock In \emph{Proceedings of the IEEE/CVF International Conference on
  Computer Vision}, pages 11463--11473, 2023.

\bibitem[Kingma and Ba(2014)]{kingma2014adam}
Diederik~P Kingma and Jimmy Ba.
\newblock Adam: A method for stochastic optimization.
\newblock \emph{arXiv preprint arXiv:1412.6980}, 2014.

\bibitem[Kirkpatrick et~al.(2017)Kirkpatrick, Pascanu, Rabinowitz, Veness,
  Desjardins, Rusu, Milan, Quan, Ramalho, Grabska-Barwinska,
  et~al.]{kirkpatrick2017overcoming}
James Kirkpatrick, Razvan Pascanu, Neil Rabinowitz, Joel Veness, Guillaume
  Desjardins, Andrei~A Rusu, Kieran Milan, John Quan, Tiago Ramalho, Agnieszka
  Grabska-Barwinska, et~al.
\newblock Overcoming catastrophic forgetting in neural networks.
\newblock \emph{Proceedings of the National Academy of Sciences}, pages
  3521--3526, 2017.

\bibitem[Krizhevsky(2009)]{krizhevsky2009learning}
A Krizhevsky.
\newblock Learning multiple layers of features from tiny images.
\newblock \emph{Master's thesis, University of Tront}, 2009.

\bibitem[Lester et~al.(2021)Lester, Al{-}Rfou, and
  Constant]{DBLP:conf/emnlp/LesterAC21}
Brian Lester, Rami Al{-}Rfou, and Noah Constant.
\newblock The power of scale for parameter-efficient prompt tuning.
\newblock In \emph{Proceedings of the Conference on Empirical Methods in
  Natural Language Processing}, pages 3045--3059, 2021.

\bibitem[Li et~al.(2019)Li, Zhou, Wu, Socher, and Xiong]{li2019learn}
Xilai Li, Yingbo Zhou, Tianfu Wu, Richard Socher, and Caiming Xiong.
\newblock Learn to grow: A continual structure learning framework for
  overcoming catastrophic forgetting.
\newblock In \emph{Proceedings of the International Conference on Machine
  Learning}, pages 3925--3934, 2019.

\bibitem[Li and Liang(2021)]{DBLP:conf/acl/LiL20}
Xiang~Lisa Li and Percy Liang.
\newblock Prefix-tuning: Optimizing continuous prompts for generation.
\newblock In \emph{Proceedings of the Annual Meeting of the Association for
  Computational Linguistics}, pages 4582--4597, 2021.

\bibitem[Liang and Li(2023{\natexlab{a}})]{DBLP:conf/nips/LiangL23}
Yan{-}Shuo Liang and Wu{-}Jun Li.
\newblock Loss decoupling for task-agnostic continual learning.
\newblock In \emph{Advances in Neural Information Processing Systems},
  2023{\natexlab{a}}.

\bibitem[Liang and Li(2023{\natexlab{b}})]{liang2023adaptive}
Yan-Shuo Liang and Wu-Jun Li.
\newblock Adaptive plasticity improvement for continual learning.
\newblock In \emph{Proceedings of the IEEE/CVF Conference on Computer Vision
  and Pattern Recognition}, pages 7816--7825, 2023{\natexlab{b}}.

\bibitem[Lin et~al.(2022)Lin, Yang, Fan, and Zhang]{lin2021trgp}
Sen Lin, Li Yang, Deliang Fan, and Junshan Zhang.
\newblock Trgp: Trust region gradient projection for continual learning.
\newblock In \emph{International Conference on Learning Representations}, 2022.

\bibitem[Mahabadi et~al.(2021)Mahabadi, Henderson, and
  Ruder]{DBLP:conf/nips/MahabadiHR21}
Rabeeh~Karimi Mahabadi, James Henderson, and Sebastian Ruder.
\newblock Compacter: Efficient low-rank hypercomplex adapter layers.
\newblock In \emph{Advances in Neural Information Processing Systems}, pages
  1022--1035, 2021.

\bibitem[Masana et~al.(2021)Masana, Van~de Weijer, Twardowski,
  et~al.]{masana2021importance}
Marc Masana, Joost Van~de Weijer, Bart{\l}omiej Twardowski, et~al.
\newblock On the importance of cross-task features for class-incremental
  learning.
\newblock \emph{arXiv preprint arXiv:2106.11930}, 2021.

\bibitem[Parisi et~al.(2019)Parisi, Kemker, Part, Kanan, and
  Wermter]{parisi2019continual}
German~I Parisi, Ronald Kemker, Jose~L Part, Christopher Kanan, and Stefan
  Wermter.
\newblock Continual lifelong learning with neural networks: A review.
\newblock \emph{Neural Networks}, pages 54--71, 2019.

\bibitem[Peng et~al.(2019)Peng, Bai, Xia, Huang, Saenko, and
  Wang]{peng2019moment}
Xingchao Peng, Qinxun Bai, Xide Xia, Zijun Huang, Kate Saenko, and Bo Wang.
\newblock Moment matching for multi-source domain adaptation.
\newblock In \emph{Proceedings of the IEEE/CVF International Conference on
  Computer Vision}, pages 1406--1415, 2019.

\bibitem[Rusu et~al.(2016)Rusu, Rabinowitz, Desjardins, Soyer, Kirkpatrick,
  Kavukcuoglu, Pascanu, and Hadsell]{rusu2016progressive}
Andrei~A Rusu, Neil~C Rabinowitz, Guillaume Desjardins, Hubert Soyer, James
  Kirkpatrick, Koray Kavukcuoglu, Razvan Pascanu, and Raia Hadsell.
\newblock Progressive neural networks.
\newblock \emph{arXiv preprint arXiv:1606.04671}, 2016.

\bibitem[Saha et~al.(2021)Saha, Garg, and Roy]{DBLP:conf/iclr/SahaG021}
Gobinda Saha, Isha Garg, and Kaushik Roy.
\newblock Gradient projection memory for continual learning.
\newblock In \emph{International Conference on Learning Representations}, 2021.

\bibitem[Smith et~al.(2023{\natexlab{a}})Smith, Hsu, Zhang, Hua, Kira, Shen,
  and Jin]{DBLP:journals/corr/abs-2304-06027}
James~Seale Smith, Yen{-}Chang Hsu, Lingyu Zhang, Ting Hua, Zsolt Kira, Yilin
  Shen, and Hongxia Jin.
\newblock Continual diffusion: Continual customization of text-to-image
  diffusion with c-lora.
\newblock \emph{CoRR}, 2023{\natexlab{a}}.

\bibitem[Smith et~al.(2023{\natexlab{b}})Smith, Karlinsky, Gutta,
  Cascante-Bonilla, Kim, Arbelle, Panda, Feris, and Kira]{smith2023coda}
James~Seale Smith, Leonid Karlinsky, Vyshnavi Gutta, Paola Cascante-Bonilla,
  Donghyun Kim, Assaf Arbelle, Rameswar Panda, Rogerio Feris, and Zsolt Kira.
\newblock Coda-prompt: Continual decomposed attention-based prompting for
  rehearsal-free continual learning.
\newblock In \emph{Proceedings of the IEEE/CVF Conference on Computer Vision
  and Pattern Recognition}, pages 11909--11919, 2023{\natexlab{b}}.

\bibitem[Sun et~al.(2022)Sun, Lyu, Shang, Feng, and Wan]{sun2022exploring}
Qing Sun, Fan Lyu, Fanhua Shang, Wei Feng, and Liang Wan.
\newblock Exploring example influence in continual learning.
\newblock \emph{Advances in Neural Information Processing Systems}, pages
  27075--27086, 2022.

\bibitem[Wang et~al.(2023{\natexlab{a}})Wang, Xie, Zhang, Huang, Su, and
  Zhu]{wang2023hierarchical}
Liyuan Wang, Jingyi Xie, Xingxing Zhang, Mingyi Huang, Hang Su, and Jun Zhu.
\newblock Hierarchical decomposition of prompt-based continual learning:
  Rethinking obscured sub-optimality.
\newblock \emph{arXiv preprint arXiv:2310.07234}, 2023{\natexlab{a}}.

\bibitem[Wang et~al.(2023{\natexlab{b}})Wang, Zhang, Su, and
  Zhu]{wang2023comprehensive}
Liyuan Wang, Xingxing Zhang, Hang Su, and Jun Zhu.
\newblock A comprehensive survey of continual learning: Theory, method and
  application.
\newblock \emph{arXiv preprint arXiv:2302.00487}, 2023{\natexlab{b}}.

\bibitem[Wang et~al.(2022{\natexlab{a}})Wang, Huang, and Hong]{wang2022s}
Yabin Wang, Zhiwu Huang, and Xiaopeng Hong.
\newblock S-prompts learning with pre-trained transformers: An occam’s razor
  for domain incremental learning.
\newblock \emph{Advances in Neural Information Processing Systems}, pages
  5682--5695, 2022{\natexlab{a}}.

\bibitem[Wang et~al.(2022{\natexlab{b}})Wang, Zhang, Ebrahimi, Sun, Zhang, Lee,
  Ren, Su, Perot, Dy, and Pfister]{DBLP:conf/eccv/0002ZESZLRSPDP22}
Zifeng Wang, Zizhao Zhang, Sayna Ebrahimi, Ruoxi Sun, Han Zhang, Chen{-}Yu Lee,
  Xiaoqi Ren, Guolong Su, Vincent Perot, Jennifer~G. Dy, and Tomas Pfister.
\newblock Dualprompt: Complementary prompting for rehearsal-free continual
  learning.
\newblock In \emph{Proceedings of the European Conference on Computer Vision},
  pages 631--648, 2022{\natexlab{b}}.

\bibitem[Wang et~al.(2022{\natexlab{c}})Wang, Zhang, Lee, Zhang, Sun, Ren, Su,
  Perot, Dy, and Pfister]{wang2022learning}
Zifeng Wang, Zizhao Zhang, Chen-Yu Lee, Han Zhang, Ruoxi Sun, Xiaoqi Ren,
  Guolong Su, Vincent Perot, Jennifer Dy, and Tomas Pfister.
\newblock Learning to prompt for continual learning.
\newblock In \emph{Proceedings of the IEEE/CVF Conference on Computer Vision
  and Pattern Recognition}, pages 139--149, 2022{\natexlab{c}}.

\bibitem[Zaken et~al.(2022)Zaken, Goldberg, and Ravfogel]{zaken2022bitfit}
Elad~Ben Zaken, Yoav Goldberg, and Shauli Ravfogel.
\newblock Bitfit: Simple parameter-efficient fine-tuning for transformer-based
  masked language-models.
\newblock In \emph{Proceedings of the Annual Meeting of the Association for
  Computational Linguistics (Short Papers)}, pages 1--9, 2022.

\bibitem[Zenke et~al.(2017)Zenke, Poole, and Ganguli]{zenke2017continual}
Friedemann Zenke, Ben Poole, and Surya Ganguli.
\newblock Continual learning through synaptic intelligence.
\newblock In \emph{Proceedings of the International Conference on Machine
  Learning}, pages 3987--3995, 2017.

\bibitem[Zhang et~al.(2021)Zhang, Bengio, Hardt, Recht, and
  Vinyals]{zhang2021understanding}
Chiyuan Zhang, Samy Bengio, Moritz Hardt, Benjamin Recht, and Oriol Vinyals.
\newblock Understanding deep learning (still) requires rethinking
  generalization.
\newblock \emph{Communications of the ACM}, pages 107--115, 2021.

\bibitem[Zhang et~al.(2023)Zhang, Wang, Kang, Chen, and
  Wei]{DBLP:conf/iccv/ZhangWKCW23}
Gengwei Zhang, Liyuan Wang, Guoliang Kang, Ling Chen, and Yunchao Wei.
\newblock {SLCA:} slow learner with classifier alignment for continual learning
  on a pre-trained model.
\newblock In \emph{Proceedings of the IEEE/CVF International Conference on
  Computer Vision}, pages 19091--19101, 2023.

\bibitem[Zheng et~al.(2023)Zheng, Ma, Wang, Qin, Yue, and
  You]{zheng2023preventing}
Zangwei Zheng, Mingyuan Ma, Kai Wang, Ziheng Qin, Xiangyu Yue, and Yang You.
\newblock Preventing zero-shot transfer degradation in continual learning of
  vision-language models.
\newblock In \emph{Proceedings of the IEEE/CVF International Conference on
  Computer Vision}, pages 19068--19079, 2023.

\bibitem[Zhou et~al.(2022)Zhou, Wei, Wang, Shen, Xie, Yuille, and
  Kong]{zhou2021image}
Jinghao Zhou, Chen Wei, Huiyu Wang, Wei Shen, Cihang Xie, Alan Yuille, and Tao
  Kong.
\newblock Image bert pre-training with online tokenizer.
\newblock In \emph{International Conference on Learning Representations}, 2022.

\bibitem[Zhu et~al.(2021)Zhu, Zhang, Wang, Yin, and Liu]{zhu2021prototype}
Fei Zhu, Xu-Yao Zhang, Chuang Wang, Fei Yin, and Cheng-Lin Liu.
\newblock Prototype augmentation and self-supervision for incremental learning.
\newblock In \emph{Proceedings of the IEEE/CVF Conference on Computer Vision
  and Pattern Recognition}, pages 5871--5880, 2021.

\end{thebibliography}
}

\clearpage

\renewcommand\thesection{\Alph{section}}
\setcounter{page}{1}
\setcounter{section}{0}
\maketitlesupplementary

\section{Details of GPM and DualGPM}\label{sec:details-of-gpm-dualgpm}
GPM and DualGPM are established on the fact that the gradient updates lie in the span of input data points~\cite{zhang2021understanding}. 

For a linear layer, we denote its forward propagation as 
\begin{align}\label{aeq3}
    \bm{e}=\bm{W}\bm{h}+\bm{b},
\end{align}
$\bm{W}\in \mathbb{R}^{d_{I}\times d_{O}}$, $\bm{h}\in\mathbb{R}^{d_{I}}$, and $\bm{e}\in\mathbb{R}^{d_{O}}$. $d_{I}$ and $d_{O}$ denote input and output dimension, respectively.
We further denote the loss function as $\mathcal{L}$. 
Through the chain rule, we can get the gradient of $\bm{W}$:
\begin{align}\label{aeq4}
    \frac{\partial \mathcal{L}}{\partial \bm{W}}&=\frac{\partial \mathcal{L}}{\partial \bm{e}}\frac{\partial \bm{e}}{\bm{W}}=\frac{\partial \mathcal{L}}{\partial \bm{e}}\bm{h}^{T}
    =\left[\begin{matrix}
        a_{1}\bm{h}^{T},\\
        a_{2}\bm{h}^{T},\\
        ...,\\
        a_{d_{O}}\bm{h}^{T}
    \end{matrix}\right].
\end{align}
$[a_{1},a_{2},...,a_{d_{O}}]^{T}$ denotes the vector $\frac{\partial \mathcal{L}}{\partial \bm{e}}$. Through~(\ref{aeq4}), we can find that each column of $\frac{\partial \mathcal{L}}{\partial \bm{W}}$ can be represented as input $\bm{h}$ multiplied by a real value $a_{k}$~($1\leq k\leq d_{O}$). Therefore, in the linear layer, each column of the gradient $\frac{\partial \mathcal{L}}{\partial \bm{W}}$ lies in the span of input. 

\subsection{Gradient Projection Memory}
\label{sec:gpm}
GPM learns a subspace ${\mathcal{M}}_{t}$ with orthogonal bases $\bm{M}_{t}$ to approximate the gradient space of the old tasks. Here, the columns of $\bm{M}_{t}$ contribute a set of orthogonal bases in $\mathcal{M}_{t}$. GPM expands the bases of ${\mathcal{M}}_{t}$ to the bases of ${\mathcal{M}}_{t+1}$ after learning the $t$-th new task. Specifically, GPM computes the inputs matrix $\bm{H}_{t}$ such that each column of $\bm{H}_{t}$ represents an input of this layer. Then, the part of $\bm{H}_{t}$ that has already in $\mathcal{M}_{t}$ is removed by 
\begin{align}\label{projection}
    \hat{\bm{H}}_{t}=\bm{H}_{t}-\bm{M}_{t}(\bm{M}_{t})^{T}\bm{H}_{t}=\bm{H}_{t}-\bm{H}_{t,proj}.
  \end{align}
Please note that when $t=1$, ${\rm dim}({\mathcal{M}}_{t})=0$ and hence $\bm{H}_{t,proj}$ is a zero matrix.
After that, singular value decomposition~(SVD) is performed on $\hat{\bm{H}}_{t}=\hat{\bm{U}}\hat{\bm{\Sigma}}\hat{\bm{V}}^{T}$. Then, $u$ new 
orthogonal bases are chosen from the columns of $\hat{\bm{U}}$ for a minimum of $u$ satisfying the following criteria for given 
threshold $\epsilon_{th}$:
\begin{align}\label{threshold in space}
    ||(\hat{\bm{H}}_{t})_{u}||_{F}^{2}+||\bm{H}_{t,proj}||_{F}^{2}\geq \epsilon_{th}||\bm{H}_{t}||_{F}^{2}.
\end{align}
Here, $(\hat{\bm{H}}_{t})_{u}=[\bm{u}_{1},...,\bm{u}_{u}]$ denotes the components of $\hat{\bm{H}}_{t}$ that correspond to top-$u$ singular values.
Then, subspace ${\mathcal{M}}_{t+1}$ is obtained with the bases $\bm{M}_{t+1}=[\bm{M}_{t},\bm{u}_{1},...,\bm{u}_{u}]$.

\subsection{Dual Gradient Projection Memory}
\label{sec:dualgpm}
Different from GPM that learns a subspace ${\mathcal{M}}_{t}$ with orthogonal bases $\bm{M}_{t}$ to approximate the gradient space of the old tasks, 
DualGPM either learns a subspace ${\mathcal{M}}_{t}$ with orthogonal bases $\bm{M}_{t}$ to approximate the gradient of the old tasks, or learns a subspace ${\mathcal{M}}_{t}^{\bot}$ with orthogonal bases $\bm{M}_{t}^{\bot}$ to approximate orthogonal complement of the gradient space of the old tasks. 

DualGPM decides whether to keep $\bm{M}_{t}$ or $\bm{M}_{t}^{\bot}$ in memory according to ${\rm dim}({\mathcal{M}}_{t})$ and ${\rm dim}({\mathcal{M}}_{t}^{\bot})$. Specifically, during the learning of the first several tasks, ${\rm dim}({\mathcal{M}}_{t})\leq{\rm dim}({\mathcal{M}}_{t}^{\bot})$. At this time, DualGPM maintains $\bm{M}_{t}$, and expands $\bm{M}_{t}$ to $\bm{M}_{t+1}$ after each task. When ${\rm dim}({\mathcal{M}}_{t})$ increases and exceeds ${\rm dim}({\mathcal{M}}_{t}^{\bot})$, DualGPM obtains $\bm{M}_{t}^{\bot}$ through some transformations on $\bm{M}_{t}$. After that, DualGPM only maintains $\bm{M}_{t}^{\bot}$ in memory, and reduces $\bm{M}_{t}^{\bot}$ to $\bm{M}_{t+1}^{\bot}$ after each task. Through this way, the number of bases kept for each layer is ${\rm min}\{{\rm dim}({\mathcal{M}}_{t}),{\rm dim}({\mathcal{M}}_{t}^{\bot})\}$. 

There are three key problems in DualGPM: expanding the bases of ${\mathcal{M}}_{t}$, 
obtaining the bases of ${\mathcal{M}}_{t}^{\bot}$ through the bases of ${\mathcal{M}}_{t}$, 
and reducing the bases of ${\mathcal{M}}_{t}^{\bot}$.

\paragraph{Expanding the Bases of ${\mathcal{M}}_{t}$}
The expansion of ${\mathcal{M}}_{t}$ is the same as that in GPM. 

\paragraph{Transforming ${\mathcal{M}}_{t}$ to ${\mathcal{M}}_{t}^{\bot}$}
DualGPM transforms ${\mathcal{M}}_{t}$ to ${\mathcal{M}}_{t}^{\bot}$
by performing SVD to the matrix $\bm{M}_{t}$. Specifically, let $\bm{M}_{t}=\bm{U}\bm{\Sigma} \bm{V}^{T}$,
the column vectors of $\bm{U}$ which correspond to the zero singular values form a set of orthogonal bases 
of ${\mathcal{M}}_{t}^{\bot}$. Please refer to the paper of DualGPM~\cite{liang2023adaptive} for this explanation.

\paragraph{Reducing the Bases of ${\mathcal{M}}_{t}^{\bot}$}
DualGPM reduces space ${\mathcal{M}}_{t}^{\bot}$ by removing the part of ${\mathcal{M}}_{t}^{\bot}$ which contains the gradient of 
the $t$-th task. Specifically, DualGPM first computes the input matrix $\bm{R}_{t}$. Then, the part of $\bm{R}_{t}$ 
which lies in ${\mathcal{M}}_{t}^{\bot}$ can be computed through 
\begin{align}\label{dual projection}
    \hat{\bm{R}}_{t}^{\bot}=\bm{M}_{t}^{\bot}(\bm{M}_{t}^{\bot})^{T}\bm{R}_{t}=\bm{R}_{t,proj}^{\bot}.
  \end{align}
After that, SVD is performed on $\hat{\bm{R}}_{t}^{\bot}=\hat{\bm{U}}^{\bot}\hat{\bm{\Sigma}}^{\bot}(\hat{\bm{V}}^{\bot})^{T}$. Then, $k$ new orthogonal bases are chosen from the columns of $\hat{\bm{U}}^{\bot}$ for a maximum of $k$ satisfying the following criteria for the given 
threshold $\epsilon_{th}$~(the same as $\epsilon_{th}$ in~(\ref{threshold in space})):
\begin{align}\label{threshold in dual space}
    ||(\hat{\bm{R}}_{t}^{\bot})_{k}||_{F}^{2}\leq (1-\epsilon_{th})||\bm{R}_{t}||_{F}^{2}.
\end{align}
Let $\bm{Z}=(\hat{\bm{R}}_{t}^{\bot})_{k}=[\bm{u}_{1}^{\bot},...,\bm{u}_{k}^{\bot}]$, $\mathcal{Z}={\rm span}\{\bm{u}_{1}^{\bot},...,\bm{u}_{k}^{\bot}\}$. Here, $\mathcal{Z}$ is the subspace of ${\mathcal{M}}_{t}^{\bot}$ that contains the gradient of the $t$-th task. 
DualGPM removes $\mathcal{Z}$ from ${\mathcal{M}}_{t}^{\bot}$ to get ${\mathcal{M}}_{t+1}^{\bot}$. Specifically, let 
$\hat{\bm{M}}_{t}^{\bot}=\bm{M}_{t}^{\bot}-\bm{Z}(\bm{Z}^{T})\bm{M}_{t}^{\bot}$. 
DualGPM performs the second SVD on $\hat{\bm{M}}_{t}^{\bot}=\widetilde{\bm{U}}^{\bot}\widetilde{\bm{\Sigma}}^{\bot}(\widetilde{\bm{V}} ^{\bot})^{T}$. The columns of $\widetilde{\bm{U}}^{\bot}$ which correspond to the non-zero singular values form the bases $\bm{M}_{t+1}^{\bot}$. Please refer to the paper of DualGPM~\cite{liang2023adaptive} for this explanation.

\begin{figure}[t]
  \centering
  \includegraphics[width=1.0\columnwidth]{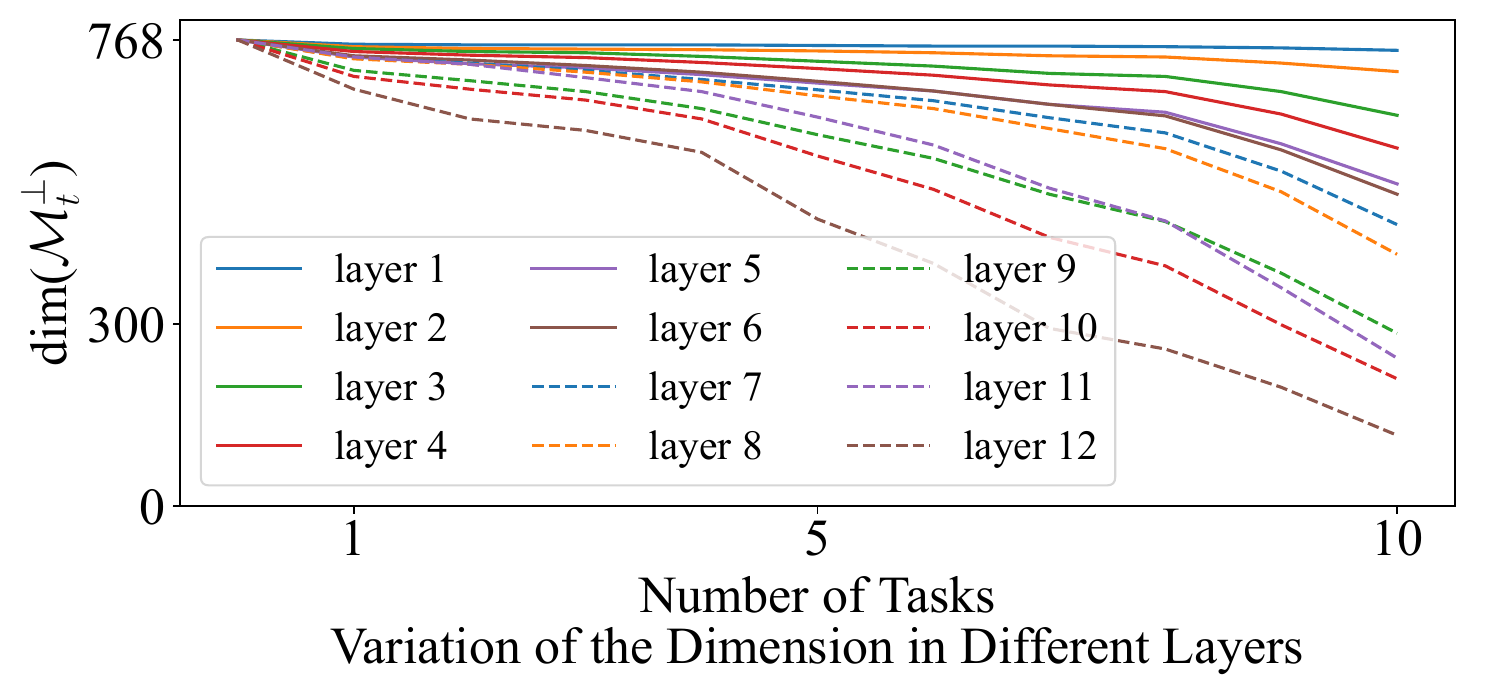}
 \captionsetup{skip=2pt}
 \caption{Change of the dimension of subspace $\mathcal{M}_{t}^{\bot}$ throughout the whole learning process.}
  \label{fig:change-dim}
  \vskip -0.1in
\end{figure}

\subsection{Approximation Error in DualGPM}\label{sec:discuss}
DualGPM either learns a subspace ${\mathcal{M}}_{t}$ to approximate the gradient space of the old tasks or learns a subspace ${\mathcal{M}}_{t}^{\bot}$ to represent the orthogonal complement of the gradient of the old tasks. From Seciton~\ref{sec:dualgpm}, we can find that the approximation error is related to the hyperparameter $\epsilon_{th}$ in~(\ref{threshold in space}) and~(\ref{threshold in dual space}). Specifically, 
when the value of $\epsilon_{th}$ in~(\ref{threshold in space}) and~(\ref{threshold in dual space}) increases, the approximation error decreases. As a result, the dimension of subspace ${\mathcal{M}}_{t}$ becomes larger, while the dimension of ${\mathcal{M}}_{t}^{\bot}$ becomes smaller. Note that our InfLoRA constrains the update of the model to lie within the subspace $\mathcal{N}_{t}\cap {\mathcal{M}}_{t}^{\bot}\subseteq {\mathcal{M}}_{t}^{\bot}$. Therefore, we can adjust the value of $\epsilon_{th}$ to adjust the space for learning the new task. 
Here, for all the experiments, we set 
\begin{align}
  \epsilon_{th}=\epsilon + \frac{(1 - \epsilon)*t}{T},
\end{align}
where $t$ denotes the task id and $T$ denotes the total number of tasks. In other words, we gradually increase the value of $\epsilon_{th}$ as the number of tasks increases throughout the whole learning process. 
Table~\ref{tab:hyperparameters} shows the setting of $\epsilon$ in our InfLoRA. 

Figure~\ref{fig:change-dim} illustrates the variation of the dimension of the subspace $\mathcal{M}_{t}^{\bot}$ in different Transformer layers of ViT-B/16. We can find that the dimension of the subspace $\mathcal{M}_{t}^{\bot}$ in different Transformer layers of ViT-B/16 is always much larger than zero, which means the space for learning the new task always exists throughout the whole learning process.




\begin{table*}[h]
  \renewcommand\arraystretch{1}
  \renewcommand\tabcolsep{5pt}
  \setlength{\belowcaptionskip}{0.1cm}
  \caption{List of hyper-parameters for different methods. The meaning of different hyperparameters is given in Section~\ref{sec:compute-expanded}. The hyperparameter $\epsilon$ in InfLoRA is explained in Section~\ref{sec:discuss}}
  \label{tab:hyperparameters}
  \small
  \centering
  \begin{tabular}{ll}
    \toprule
    Methods     & Hyper-Parameters \\
    \midrule  
    L2P       & lr: 0.001~(ImageNet-R, DomainNet, CIFAR100) \\
              & $l$: 1~(ImageNet-R, DomainNet, CIFAR100) \\
              & $p$: 30~(ImageNet-R, DomainNet, CIFAR100) \\
              & $e$: 20~(ImageNet-R, DomainNet, CIFAR100) \\
    \midrule
    DualPrompt & lr: 0.001~(ImageNet-R, DomainNet, CIFAR100) \\
               & $l_{E}$: 3~(ImageNet-R, DomainNet, CIFAR100) \\
               & $l_{S}$: 2~(ImageNet-R, DomainNet, CIFAR100) \\
               & $e_{E}$: 20~(ImageNet-R, DomainNet, CIFAR100) \\
               & $e_{S}$: 6~(ImageNet-R, DomainNet, CIFAR100) \\
    \midrule
    CODA-P       & lr: 0.001~(ImageNet-R, DomainNet, CIFAR100) \\
                 & $l$: 5~(ImageNet-R, DomainNet, CIFAR100) \\
                 & $p$: 100~(ImageNet-R, DomainNet, CIFAR100) \\
                 & $e$: 8~(ImageNet-R, DomainNet, CIFAR100) \\
    \midrule
    LAE       & lr: 0.001~(ImageNet-R, DomainNet, CIFAR100) \\
              & $r$: 5~(ImageNet-R, DomainNet, CIFAR100) \\
    \midrule
    C-LoRA    & lr: 0.001~(ImageNet-R, DomainNet, CIFAR100) \\
              & $r$: 64~(ImageNet-R, DomainNet, CIFAR100) \\
              & $\lambda$: 0.5~(ImageNet-R, DomainNet, CIFAR100) \\
    \midrule
    InfLoRA-b5 & lr: 0.001~(CIFAR100), 0.0005~(ImageNet-R, DomainNet) \\
               & $r$: 10~(ImageNet-R, CIFAR100), 20~(DomainNet) \\
               & $\epsilon$: $0.99$~(ImageNet-R), $0.95$~(CIFAR100, DomainNet) \\
    \midrule
    InfLoRA    & lr: 0.0005~(ImageNet-R, DomainNet, CIFAR100) \\
               & $r$: 10~(ImageNet-R, DomainNet, CIFAR100)\\
               & $\epsilon$: $0.98$~(ImageNet-R), $0.95$~(CIFAR100, DomainNet) \\
    \bottomrule
  \end{tabular}
\end{table*}

\section{More Experimental Details}\label{sec:hyper}
\subsection{Training Details}
For all the methods in all the experiments except for the comparison with SLCA, the batch size is set to 128 to follow many existing continual learning methods based on PEFT~\cite{smith2023coda,wang2023hierarchical}. Hyperparameters for different methods are selected based on the experimental settings in existing works~\cite{smith2023coda,wang2022learning,gao2023unified} or through hyperparameter search. For example, Adam is used as the optimizer with running averages of gradient and its square~($\beta_{1} = 0.9$, $\beta_{2} = 0.999$). The learning rate is searched among [5e-4, 1e-3, 2e-3, 1e-2] for all the methods through the validation sets we split from the training sets. For the hyperparameter $r$ in our \mbox{InfLoRA}, we search it among [1, 5, 10, 20, 30] through the validation sets we split from the training sets. Table~\ref{tab:hyperparameters} shows the hyperparameters of different methods.

When compared with SLCA, our method is combined with classifier alignment~(CA). At this time, we follow SLCA to train the expanded LoRA branches and classifiers using the SGD optimizer. Each task is trained for 50 epochs on ImageNet-R, 20 epochs on CIFAR100 and 5 epochs on DomainNet. The batch size is set to 128.



\subsection{Expanded Parameters}\label{sec:compute-expanded}
For L2P~\cite{wang2022learning}, the expanded parameters consist of the inserted prompts and their corresponding keys. Let $d$ denote the embedding dimension, $e$ denote the prompt length, $p$ denote the number of prompts, and $l$ denote the number of layers in which prompts are inserted. To compute the total number of expanded parameters, the formula used is $dlp(e+1)$.

For DualPrompt~\cite{DBLP:conf/eccv/0002ZESZLRSPDP22}, the expanded parameters also consist of the inserted prompts and corresponding keys. However, DualPrompt contains expert prompts and shared prompts. Let $d$ denote the embedding dimension, $T$ denote the number of tasks, $e_{E}$ denote the expert prompt length, $e_{S}$ denote the shared prompt length, $l_{E}$  denote the number of layers in which expert prompts are inserted and $l_{S}$ denote the number of layers in which shared prompts are inserted. To compute the total number of expanded parameters, the formula used is $d[Tl_{E}(e_{E}+1)+e_{S}l_{S}]$.

For CODA-Prompt~\cite{smith2023coda}, the expanded parameters consist of the inserted prompts, corresponding keys and attention parameters. Let $d$ denote the embedding dimension, $e$ denote the prompt length, $p$ denote the number of prompts, and $l$ denote the number of layers in which prompts are inserted. To compute the total number of expanded parameters, the formula used is $dlp(e+2)$.

For LAE~\cite{gao2023unified}, we implement it with LoRA. Therefore, the expanded parameters in this method consist of the inserted LoRA modules and the corresponding ensemble modules. Let $d$ denote the embedding dimension, $r$ denote the rank, and $l$ denote the number of layers in which LoRA modules are inserted. Since LAE inserts LoRA modules into key and value projection in multi-head attention, the number of expanded parameters is $8ldr$.

For C-LoRA~\cite{DBLP:journals/corr/abs-2304-06027}, the expanded parameters in this method consist of the inserted LoRA modules. Let $d$ denote the embedding dimension, $r$ denote the rank, and $l$ denote the number of layers in which LoRA modules are inserted. Since C-LoRA inserts LoRA modules into query, key and value projection in multi-head attention, the number of expanded parameters is $6ldr$.

For our methods, since we integrate the branches of the old tasks when the model learns a new task, the number of expanded parameters equals the number of parameters in a single branch. Let $d$ denote the embedding dimension, $r$ denote the rank, and $l$ denote the number of layers in which our InfLoRA modules are inserted. Since we also insert InfLoRA modules into key and value projection in multi-head attention, the number of expanded parameters is $4ldr$.

\section{More Experimental Results}

\begin{table*}[t]
  \caption{The comparison between our \mbox{InfLoRA} and more methods on ImageNet-R. %
  }
  \label{tab:more-compare-img}
  \centering
  \setlength{\tabcolsep}{5pt}  
  
  \begin{tabular}{l||c c||c c||c c} 
  \hline 
  Tasks & \multicolumn{2}{c||}{5} & \multicolumn{2}{c||}{10} & \multicolumn{2}{c}{20} \\
  \hline
  \rule{0pt}{10pt} Method & $ACC_5$ ($\uparrow$) & $\overline{ACC}_5$ ($\uparrow$) & $ACC_{10}$ ($\uparrow$) & $\overline{ACC}_{10}$ ($\uparrow$) & $ACC_{20}$ ($\uparrow$) & $\overline{ACC}_{20}$ ($\uparrow$) \\
  \hline
  SeqLoRA
  & $70.96 \pm 0.25$ & $79.14 \pm 0.32$ 
  & $64.32 \pm 0.09$ & $74.78 \pm 0.29$
  & $56.98 \pm 0.29$ & $69.29 \pm 0.26$ \\
  HiDe-Prompt~\cite{wang2023hierarchical} 
  & $76.82 \pm 0.91$ & $77.19 \pm 0.34$ 
  & $75.06 \pm 0.12$ & $76.60 \pm 0.01$ 
  & $66.88 \pm 1.29$ & $76.71 \pm 0.23$ \\
  \mbox{InfLoRA}
  & \textbf{77.52 $\pm$ 0.37} & \textbf{82.01 $\pm$ 0.12} 
  & \textbf{75.65 $\pm$ 0.14} & \textbf{80.82 $\pm$ 0.24} 
  & \textbf{71.01 $\pm$ 0.45} & \textbf{77.28 $\pm$ 0.45}  \\ 
  \hline
  \end{tabular}
  \vskip -0.1in
\end{table*}

\begin{table}[t]
  \caption{The comparison between our \mbox{InfLoRA} and more methods on DomainNet. %
  }
  \label{tab:more-compare-domain}
  \centering
  \setlength{\tabcolsep}{5pt}  
  
  \begin{tabular}{c||c c} 
  \hline
  \rule{0pt}{10pt} & $ACC_5$ ($\uparrow$) & $\overline{ACC}_5$ ($\uparrow$) \\
  \hline
  SeqLoRA 
  & $71.69 \pm 0.13$ & $78.68 \pm 0.12$ \\
  HiDe-Prompt~\cite{wang2023hierarchical} 
  & $71.48 \pm 0.10$ & $76.15 \pm 0.05$ \\
  \mbox{InfLoRA}
  & \textbf{74.53 $\pm$ 0.23} & \textbf{79.57 $\pm$ 0.57} \\
  \hline
  \end{tabular}
\end{table}

\begin{table}[t]
    \caption{Results (\%) of different methods on ImageNet-R~(10 tasks) using various self-supervised pre-trained models. Here, DINO-1k and iBOT-1k indicate that the ViT is pre-trained on ImageNet-1k using these respective methods. %
    }
    \label{tab:add-cifar100_vary_pre-trained}
    \centering
    \setlength{\tabcolsep}{2pt}  
    
    \begin{tabular}{c|l||c c} 
    \hline
    \rule{0pt}{10pt} & Method & $ACC_{10}$ ($\uparrow$) & $\overline{ACC}_{10}$ ($\uparrow$) \\
    \hline
    \multirow{4}{*}{DINO-1k} & 
    SeqLoRA 
    & $60.67 \pm 0.11$ & $66.29 \pm 0.21$ \\
    & HiDe-Prompt~\cite{wang2023hierarchical} 
    & $68.11 \pm 0.18$ & $71.70 \pm 0.01$ \\
    & \mbox{InfLoRA}
    & \textbf{68.31 $\pm$ 0.28} & \textbf{76.15 $\pm$ 0.05}  \\
    \hline
    \multirow{4}{*}{iBOT-1k} & 
    SeqLoRA 
    & $66.87 \pm 0.40$ & $71.80 \pm 0.28$ \\
    & HiDe-Prompt~\cite{wang2023hierarchical} 
    & $71.33 \pm 0.21$ & $73.62 \pm 0.13$ \\
    & \mbox{InfLoRA}
    & \textbf{71.84 $\pm$ 0.09} & \textbf{78.29 $\pm$ 0.09} \\
    \hline
    \end{tabular}
  \end{table}

\subsection{Compare with More Methods}
We compare with SeqLoRA, which initials LoRA modules and finetunes these modules on multiple tasks sequentially without any operation to overcome forgetting. The results are given in Table~\ref{tab:more-compare-img}, Table~\ref{tab:more-compare-domain} and Table~\ref{tab:add-cifar100_vary_pre-trained}. We can find that our method outperforms this method.

A recent continual learning PEFT method hierarchical decomposition prompt~(HiDe-Prompt)~\cite{wang2023hierarchical} proposes to perform continual learning hierarchically. This method maintains a set of task-specific prompts for each task and contains two stages during training and inference. Specifically, given an input sample, Hide-Prompt infers the prompt index and then uses the corresponding prompt to infer its label. We also compare our method with this method, and the results are also given in Table~\ref{tab:more-compare-img}, Table~\ref{tab:more-compare-domain} and Table~\ref{tab:add-cifar100_vary_pre-trained}. We can find that our method outperforms this method. Furthermore, this method shows comparable performance to our method in terms of final accuracy $ACC_{T}$ on ImageNet-R. However, there is a notable gap between this method and our method in terms of averaged accuracy $\overline{ACC}_{T}$. Note that averaged accuracy $\overline{ACC}_{T}$ is more important than final accuracy $ACC_{T}$ since $\overline{ACC}_{T}$ represents the performance of the model over the whole learning process.


\subsection{Hyperparameter Analysis}
We perform the hyperparameter analysis for our method \mbox{InfLoRA}. There are two specific hyperparameters in our method \mbox{InfLoRA}. The first hyperparameter is $r$, which controls the expanded parameters in \mbox{InfLoRA}. The second hyperparameter is $\epsilon$, which is not the specific hyperparameter of our \mbox{InfLoRA} but the hyperparameter introduced by DualGPM. This hyperparameter controls the component maintained in the matrix $\bm{M}_{t}$.

Figure~\ref{fig:dif-hyper} shows the results of our method with different values of $r$ or $\epsilon$. We can find that the performance of InfLoRA increases first and then decreases with the increase of $r$ and $\epsilon$.

\begin{figure}
    \centering
    \includegraphics[width=1.0\columnwidth]{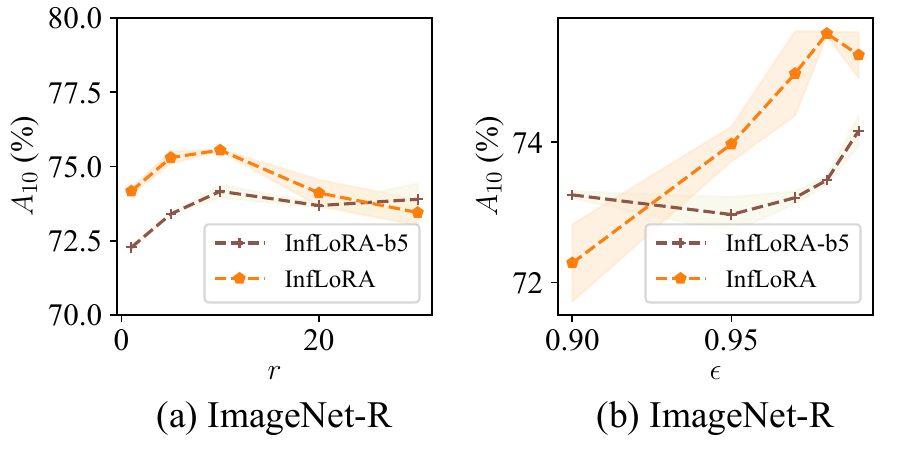}
    \caption{(a) Analysis of the hyperparameter $r$.~(b) Analysis of the hyperparameter $\epsilon$.}
    \label{fig:dif-hyper}
  \end{figure}

\subsection{Domain Incremental Setting}
InfLoRA can be extended to the domain incremental setting. Specifically, DomainNet contains six domains and InfLoRA can learn on these domains sequentially. Table~\ref{tab:domain} shows that InfLoRA outperforms other baselines.


\begin{table}[t]
  \caption{Results of DomainNet for domain incremental setting.}
  \label{tab:domain}
  \centering
  \setlength{\tabcolsep}{5pt}  
  
  \begin{tabular}{l||c c} 
  \hline
  \rule{0pt}{10pt} Method & $ACC_{6}$ ($\uparrow$) & $\overline{ACC}_{6}$ ($\uparrow$) \\
  \hline
  L2P~\cite{wang2022learning} 
  & $34.15 \pm 0.10$ & $49.84 \pm 0.03$ \\
  DualPrompt~\cite{DBLP:conf/eccv/0002ZESZLRSPDP22} 
  & $35.24 \pm 0.12$ & $48.44 \pm 0.13$ \\
  CODA-P~\cite{smith2023coda}   
  & $56.89 \pm 0.04$ & $57.56 \pm 0.03$	\\ 
  C-LoRA~\cite{DBLP:journals/corr/abs-2304-06027} 
  & $44.96 \pm 0.01$ & $52.95 \pm 0.08$ \\
  \mbox{InfLoRA}
  & \textbf{\textbf{68.44 $\pm$ 0.04}} & \textbf{67.46 $\pm$ 0.03} \\
  \hline
  \end{tabular}
  \vskip -0.1in
\end{table}

\subsection{Inference Efficiency}
Existing methods often involve multiple forward propagations through the pre-trained backbone. Specifically, prompt-based continual learning methods, including L2P, DualPrompt, and CODA-P, require an extra forward propagation to generate instance-specific prompts. LAE requires an extra forward propagation for ensembling. In contrast, our \mbox{InfLoRA} only requires a single forward propagation through the pre-trained backbone. Figure~\ref{fig:time} provides a comparison of the time consumed by different methods during inference. We can find that our method consistently outperforms existing methods in terms of time efficiency.

\begin{figure}[t]
  \centering
  \includegraphics[width=0.93\columnwidth]{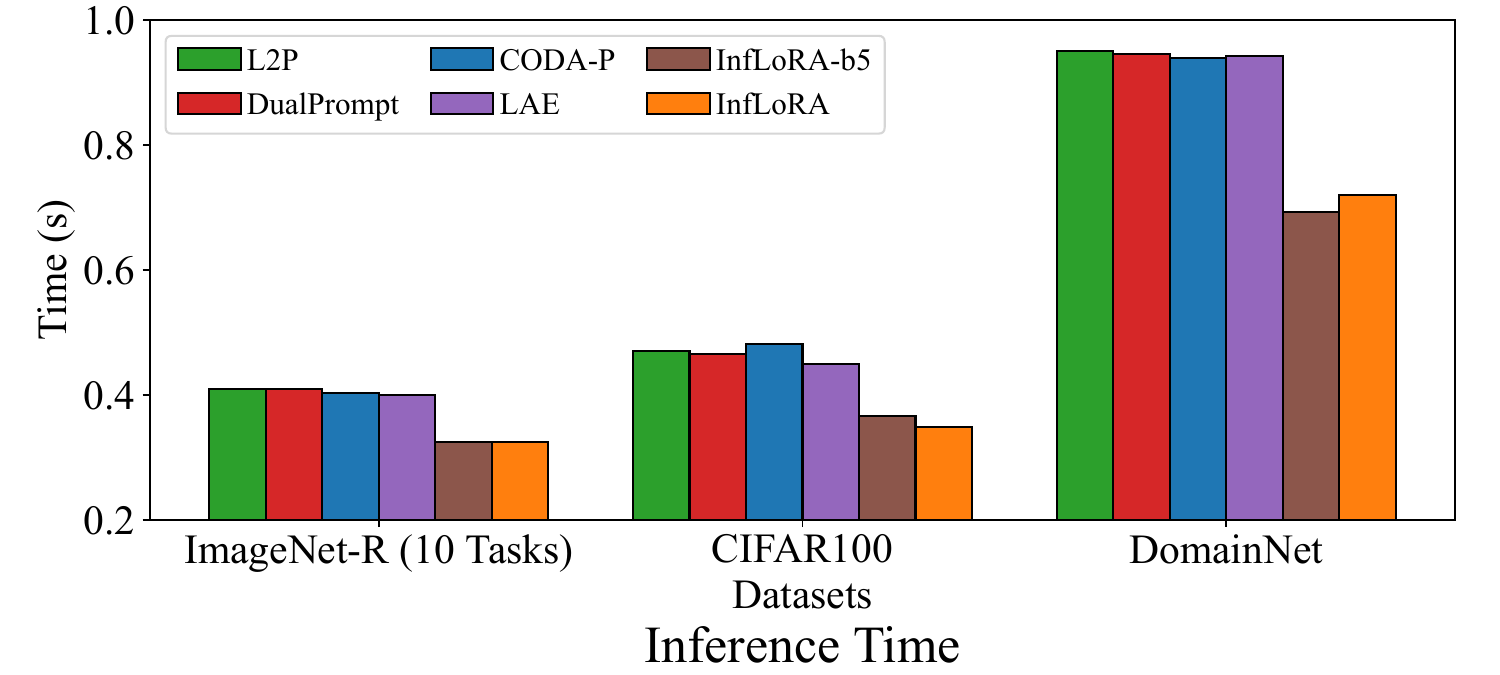}
  \caption{The time of inferring one task for different methods.}
  \label{fig:time}
  \vskip -0.15in
\end{figure}

\end{document}